\documentclass{article}

\PassOptionsToPackage{numbers, compress}{natbib}

\usepackage[preprint]{neurips_2023}

\usepackage{multirow}
\usepackage{xcolor,colortbl}
\usepackage{multirow}
\usepackage{hhline}
\usepackage{colortbl}
\usepackage{multicol}
\usepackage{bbm}
\usepackage{enumitem}
\usepackage{soul}
\usepackage[english]{babel}
\usepackage{amsthm}
\usepackage{algorithm}
\usepackage{algpseudocode}

\usepackage{hhline}
\usepackage{longtable}
\usepackage{amsmath}

\usepackage[utf8]{inputenc} 
\usepackage[T1]{fontenc} 
\usepackage{hyperref} 
\usepackage{url} 
\usepackage{booktabs} 
\usepackage{amsfonts} 
\usepackage{nicefrac} 
\usepackage{microtype} 
\usepackage{xcolor} 
\usepackage{svg}
\usepackage{verbatim}

\newcommand{\ignore}[1]{}

\newtheorem{theorem}{Theorem}[section]
\newtheorem{lemma}[theorem]{Lemma}

\theoremstyle{definition}
\newtheorem{definition}{Definition}[section]

\DeclareMathOperator*{\argmax}{argmax} 

\definecolor{lightBlueTable}{RGB}{214,230,244} 
\definecolor{lightPeachTable}{RGB}{248,210,200}
\definecolor{lightYellowTable}{RGB}{255,240,193}
\definecolor{lightGreenTable}{RGB}{220,237,208}
\definecolor{lightGrayTable}{RGB}{239,239,239}

\title{\ignore{Conditional Bias Scan: }Auditing Predictive Models for Intersectional Biases}

\author{%
 Kate S. Boxer \\
 Machine Learning for Good Laboratory\\
 New York University\\
 \texttt{kb145@nyu.edu} \\
 \And
 Edward McFowland III \\
 Harvard Business School\\
 \texttt{emcfowland@hbs.edu} \\
 \And
 Daniel B. Neill\\
 Machine Learning for Good Laboratory\\
 New York University\\
 \texttt{daniel.neill@nyu.edu} \\
}

\begin{document}

\maketitle

\begin{abstract}
Predictive models that satisfy group fairness criteria in aggregate for members of a protected class, but do not guarantee subgroup fairness, could produce biased predictions for individuals at the intersection of two or more protected classes. To address this risk, we propose Conditional Bias Scan (CBS), a flexible auditing framework for detecting intersectional biases in classification models. CBS identifies the subgroup for which there is the most significant bias against the protected class, as compared to the equivalent subgroup in the non-protected class, and can incorporate multiple commonly used fairness definitions for both probabilistic and binarized predictions.
We show that this methodology can detect previously unidentified intersectional and contextual biases in the COMPAS pre-trial risk assessment tool and has higher bias detection power compared to similar methods that audit for subgroup fairness.
\end{abstract}




\maketitle

\section{Introduction}\label{sec:intro}

Predictive models are increasingly used to assist in high-stakes decisions with significant impacts on individuals' lives and livelihoods. However, recent studies have revealed numerous models whose predictions contain biases, in the form of group fairness violations, against disadvantaged and marginalized groups~\citep{angwin_larson_kirchner_2016, obermeyer2019dissecting}. \ignore{Various mathematical formulations have been proposed to measure group fairness, that is, whether the outputs of a predictive model are biased against a particular ``protected class'', typically defined by a sensitive attribute such as race or gender. In this paper, we address several limitations of typical group fairness frameworks. 
First, commonly used definitions of fairness based on \emph{separation} (e.g., error rate balance) and \emph{sufficiency} (e.g., calibration) are incompatible, meaning that predictions which satisfy one definition will fail to satisfy the other, in contexts where the prevalence of the outcome of interest is not equal across different subpopulations \cite{chouldechova2017fair, feller2016computer,kleinberg2016inherent}. Thus we develop a flexible auditing framework which can incorporate multiple fairness definitions, including those based both on sufficiency and separation.}
When auditing a predictive model for bias, typical group fairness definitions~\citep{mitchell2021algorithmic} rely on univariate measurements of the difference between the distributions of predictions or outcomes for individuals in a ``protected class'', typically defined by a sensitive attribute such as race or gender, as compared to those in the non-protected class.  Since these approaches only detect biases for a predetermined subpopulation at an aggregate level, e.g., a bias against Black individuals, they may fail to detect biases that adversely affect a subset of individuals in a protected class, e.g., Black females.\ignore{ This shortcoming of group fairness definitions has been documented in commercial gender classification systems~\citep{buolamwini2018gender}.}\ignore{Within the social sciences, this phenomenon is addressed by using analytic frameworks that recognize \emph{intersectionality}, meaning that an individual's experience of systematic disadvantage is dependent on a combination of their race, gender, socioeconomic status, age, and other factors~\citep{runyan2018intersectionality}.} While it is possible to define a specific multidimensional subgroup and then audit a classifier for biases impacting that subgroup, this approach does not scale to the combinatorial number of subgroups. Therefore, group fairness measurements cannot reliably detect if there are \emph{any} subgroups within a given population that are adversely impacted by predictive biases, and as a result, intersectional biases in model predictions often go unaddressed.

\ignore{
OUTLINE OF PAPER FOR REWRITE
1. ML models are used to more settings
2. These models have been studied and shown to contain biases
3. biases are commonly defined in terms of group fairness defintions
4. here is a limitation of groupfairness}

In this paper, we present a novel methodology for bias detection called Conditional Bias Scan (CBS). Given 
a classifier's probabilistic \emph{predictions} \ignore{(e.g., defendants' predicted risk of reoffending) }or binarized \emph{recommendations} based on those predictions\ignore{ (e.g., defendants whose predicted risk exceeds a threshold are classified as ``high risk'')}, CBS discovers systematic biases impacting any \emph{subgroups} of a predefined subpopulation of interest (the \emph{protected class}). 
More precisely, CBS aims to discover subgroups of the protected class for whom the classifier's predictions or recommendations systematically deviate from the corresponding subgroup of individuals who are not a part of the protected class. Subgroups are defined by a non-empty subset of attribute values for each observed attribute, excluding the \emph{sensitive attribute} which determines whether or not individuals belong to the protected class. \ignore{CBS can flexibly audit classifiers for bias in both predictions and recommendations, applying multiple commonly used definitions of group fairness. }The detected subgroups can represent both \emph{intersectional} biases, defined by membership in two or more protected classes,\ignore{ e.g., biases against Black males in criminal justice risk assessment,} as well as \emph{contextual} biases that may only be present for certain decision situations\ignore{, like assessing risk for a first-time offender}~\citep{runyan2018intersectionality}. \ignore{We hope that this work will aid practitioners in detecting hard-to-observe biases that might greatly impact individuals, and address the lack of fairness auditing tools to detect intersectional and contextual subgroup biases in classifier outputs. }

The contributions of our research include:
\begin{itemize} [noitemsep,topsep=0pt,parsep=0pt,partopsep=0pt,leftmargin=5mm]
 \item The first methodological framework that can reliably detect intersectional and contextual biases and flexibly accommodates multiple common group-fairness definitions.
 \item A novel and computationally efficient pattern detection algorithm to audit classifiers for fairness violations in the exponentially many subgroups of a prespecified protected class.
 \ignore{\item CBS incorporates novel pattern detection methodology to efficiently audit classifiers for fairness violations in the exponentially many subgroups defined across multiple data dimensions. }
 \item An interpretable tool for detecting fairness violations that demonstrates significantly improved bias detection accuracy, as compared to other tools used to audit classifiers for subgroup fairness.
 \item A robust simulation study and real-world case study that compare results across various group-fairness metrics, demonstrating substantial differences between separation and sufficiency metrics.
\end{itemize}
\ignore{
In Section~\ref{related_work} we review related work. We explain the CBS framework in Section~\ref{methods}. We present simulations designed to evaluate our method, and demonstrate that it has better bias detection power than similar methodologies that could be used to audit for subgroup biases, in Section~\ref{experiments_validation}. In Section~\ref{compas_experiment_section} we show that CBS can detect previously undiscovered intersectional and contextual biases in COMPAS data. We close with discussion and concluding remarks in Section~\ref{discussion_and_limitations}. 
}
\section{Related Work} \label{related_work}

\ignore{
OUTLINE OF paragraph below for editing section for editing
1. uses the same MD subset scan for scan many subgroups
2. singular fairness definition -- calibration of predictions
3. Does not allow for specifics of protected class 
3.1 therefore CBS compares protected to non-protected 
4. here is a concrete example of what CBS can do that BS cannot
5. Here is all the technical functionality that had to be added to do this.
}

The original Bias Scan proposed by Zhang and Neill~\cite{zhang2016identifying} uses a multidimensional subset scan to search exponentially many subgroups of data, identifying the subgroup with the most significantly miscalibrated probabilistic predictions as compared to the observed outcomes.  Bias Scan lacks the functionality of traditional group fairness techniques to define a protected class and to determine whether those individuals are impacted by biased predictions, and is thus limited to asking, ``Which subgroup has the most miscalibrated predictions?"  In contrast, given a protected class $A$, CBS can reliably identify biases impacting $A$ or any subgroup of $A$.   
CBS searches for subgroups within the protected class with the most significant deviation in their predictions and observed outcomes as compared to the predictions and observed outcomes for the corresponding subgroup of the non-protected class (for example, a racial bias against Black females as compared to non-Black females).  Since Bias Scan solely focuses on the deviation between the predictions and observed outcomes within a subgroup, it would not be able to detect such a bias unless the subgroup was also biased as compared to the population as a whole. Furthermore, CBS generalizes to both separation-based and sufficiency-based group fairness metrics, and to both probabilistic and binarized predictions.  To enable these new functionalities, CBS deviates from Bias Scan's methodology in substantial ways including new preprocessing techniques, new and altered hypotheses, and resulting score functions.

GerryFair and Multiaccuracy Boost~\citep{kearns2018preventing, kim2019multiaccuracy} are two methods that use an auditor to iteratively detect subgroups while training or correcting a classifier to guarantee subgroup fairness. 
GerryFair's auditor relies on linear regressions trained to predict differences between the predictions and the observed global error rate of a dataset. Multiaccuracy Boost iteratively forms subgroups by evaluating rows with predictions above and below a threshold to determine which predictions to adjust. CBS's methodology for forming subgroups is more complex because it does not assume a linear relationship between covariates and the difference between the predictions and baseline error rate. Unlike CBS, these methods provide limited fairness definitions to use for auditing, and do not return interpretable subgroups that are defined by discrete attribute values of the covariates, but rather identify all rows that have a fairness violation on a given iteration. Since both methods incorporate the predictions in forming subgroups and enable auditing, they are directly comparable to CBS. In Section~\ref{experiments_validation}, we show that CBS has substantially higher bias detection accuracy than GerryFair and Multiaccuracy Boost.

\section{Methods} \label{methods}
\ignore{
\begin{figure}
\centering
\includesvg[scale=.37]{figures_images/cbs_overview_figure/CBS_framework 02_05_2023_v8.svg}
\caption{Overview of the CBS framework. The input data set $D = (A, X, Y, P, P_{bin})$. Event variable $I$ and conditional variable $C$ are each defined as one of $Y$, $P$, or $P_{bin}$, depending on the fairness definition. Expectations $\hat I = E[I \:|\: C, X; H_0]$ are estimated for each individual in the protected class ($A=1$), under the null hypothesis that $I \perp A \:|\: (C, X)$, and thus that expectations can be estimated from the non-protected class ($A=0$). Finally, CBS scans for subgroup $S^\ast$ of the protected class where $\hat I$ deviates most from $I$.}
\label{methods-overview_of_framework-framework_chart}
\end{figure}}

\ignore{The CBS framework audits predictions or recommendations--from a machine learning classifier, a human decision maker, or a combination of both--for systematic bias against subgroups of a protected class.} CBS begins by defining the dataset $D = (A,X,Y,P,P_{bin}) = \{(A_i, X_i, Y_i, P_i, P_{i,bin})\}_{i=1}^n$, for $n$ individuals indexed as $i=1..n$. $A_i$ is a binary variable representing whether individual $i$ belongs to the protected class. \ignore{ For example, if we are auditing COMPAS for racial biases, $A_i$ could equal 1 for Black defendants and 0 for other racial groups.}$X_i = (X_i^1 \ldots X_i^m)$ are other covariates for individual $i$\ignore{, such as age group, gender, and charge degree for the COMPAS example}. We assume that all covariates are discrete-valued; continuous covariates can be discretized as a preprocessing step. The protected class attribute $A$ and the sensitive attribute of the data from which $A$ was derived\ignore{ (e.g., individuals' race for our COMPAS example)} are not included in $X$. $Y_i$ is individual $i$'s observed binary outcome\ignore{ (e.g., were they rearrested within a 2-year follow-up period)}, $P_i \in [0,1]$ is the classifier's probabilistic prediction of individual $i$'s outcome\ignore{ (e.g., predicted probability of rearrest)}, and $P_{i,bin} \in \{0,1\}$ is the binary recommendation corresponding to $P_i$\ignore{ (e.g., predict ``high risk'' if $P_i$ exceeds some threshold)}.
\ignore{ outline
1. introduce figure x
2. introduce sections}
Given these data, CBS searches for subgroups of the protected class, defined by a non-empty subset of values for each covariate $X^1 \ldots X^m$, for whom some \emph{group fairness definition} (contained in Table~\ref{methods-overview_of_scans_scan_table}) is violated. Each fairness definition can be viewed as a conditional independence relationship between an individual's membership in the protected class $A_i$ and their value of an \emph{event variable} $I_i$, conditioned on their covariates $X_i$ and their value of a \emph{conditional variable} $C_i$.  We define the null hypothesis, $H_0$, that $I \perp A \:|\: (C, X)$, and use CBS to search for subgroups with statistically significant violations of this conditional independence relationship, correctly adjusting for multiple hypothesis testing, allowing us to reject $H_0$ for the alternative hypothesis $H_1$ that  $I \not\perp A \:|\: (C, X)$ . 

The CBS framework has four sequential steps. (1) Given a fairness definition, CBS chooses $I \in \{Y, P, P_{bin}\}$ and $C \in \{Y, P, P_{bin}\}$. In Section~\ref{methods-overview_of_scan_types}, we describe how 
different group fairness criteria in the literature map to particular choices of event variable $I$ and conditional variable $C$. (2) Next, we estimate the expected value of $I_i$ for each individual in the protected class under the null hypothesis $H_0$ that $I$ and $A$ are conditionally independent\ignore{, and thus, that this expectation can be estimated from individuals in the non-protected class}. We denote these expectations as $\hat I_i$. Our procedure to estimate $\hat{I}$, which builds on the econometric literature for estimation of heterogeneous treatment effects, is described in Section~\ref{methods-generating-hat-p-section}. (3) Then we use a novel multidimensional subset scan algorithm to search for subgroups $S$ where for $i \in S$, $I_i$ deviates systematically from its expectation $\hat I_i$ in the direction of interest. \ignore{To do so, we define a log-likelihood ratio scan statistic $F(S)$ and use fast subset scanning approaches to efficiently identify the subgroup which maximizes $F(S)$. }This scan step to \emph{detect $S^{*}$} is described in detail in Section~\ref{methods_cbs_score_functions_altss_section}. (4) The final step to \emph{evaluate statistical significance} of the detected subgroup $S^\ast$ (Section~\ref{methods_cbs_score_functions_altss_section}) involves using permutation testing, described in Appendix~\ref{permutation_testing}, to adjust for multiple hypothesis testing and determine if $S^{*}$'s deviation between protected class and non-protected class is statistically significant.

\subsection{Define $(I,C)$: \textit{Overview of Scan Types}} \label{methods-overview_of_scan_types}
Many of the group fairness criteria proposed in the fairness literature fall into two categories of statistical fairness called sufficiency and separation. \ignore{Sufficiency refers to the notion that individuals with the same recommendation or probability of a given outcome conditional on their attributes have the same probabilities of a positive outcome independent of their membership in a protected class. Conversely, separation refers to the notion that individuals with the same outcome conditional of their attributes should have the same probability of receiving a given recommendation or probability independent of their membership in a protected class \cite{mitchell2021algorithmic}.} An intuitive way to distinguish between these concepts is that \emph{sufficiency} is focused on equivalency in the rate of an outcome (for comparable individuals with the same prediction or recommendation) regardless of protected class membership ($Y \perp A \:|\: P, X$), whereas \emph{separation} is focused on equivalency of the expected prediction or recommendation (for comparable individuals with the same outcome) regardless of protected class membership ($P \perp A \:|\: Y, X$). 
In our CBS framework, the choice between separation and sufficiency determines whether outcome $Y$ is the event variable of interest $I$ or the conditional variable $C$, where bias exists if $\mathbb{E}[I\:|\:C,X, A = 1] \ne \mathbb{E}[I\:|\:C,X, A=0]$. The combination of fairness metric (sufficiency or separation) and prediction type (continuous prediction or binary recommendation) produces four general classes of fairness scans: separation for predictions ($I = P$, $C = Y$), separation for recommendations ($I = P_{bin}$, $C = Y$), sufficiency for predictions ($I = Y$, $C = P$), and sufficiency for recommendations ($I = Y$, $C=P_{bin}$).
\begin{table}
\scalebox{.68}{
\begin{tabular}{ |p{1.5cm}|p{1cm}|p{4cm}|p{4cm}|p{4cm}|p{3.5cm}|}
\hline
 \multicolumn{2}{|l|}{}& \cellcolor{lightBlueTable} Predictions ($P \in [0,1]$) & \multicolumn{3}{l|}{ \cellcolor{lightPeachTable} Recommendations ($P_{bin} \in \{0,1\}$)} \\\cline{4-6} \hhline{~~~---}
 \multicolumn{2}{|l|}{} & \cellcolor{lightBlueTable} & \cellcolor{lightPeachTable} $P_{bin} =1$ & \cellcolor{lightPeachTable} $P_{bin} =0$ & \cellcolor{lightPeachTable} $P_{bin}$ \\
 \hline
 
 \cellcolor{lightYellowTable} & \cellcolor{lightYellowTable} $Y=1$ & \cellcolor{lightGrayTable} $\mathbb{E}[P\:|\:Y=1,X] \bot A$ & $\mbox{Pr}(P_{bin}=1\:|\:Y=1,X) \bot A$ & $\mbox{Pr}(P_{bin}=0\:|\:Y=1,X) \bot A$ & \\
 \cellcolor{lightYellowTable} & \cellcolor{lightYellowTable} & \cellcolor{lightGrayTable} \emph{Balance for Positive Class} & \emph{True Positive Rate} & \emph{False Negative Rate} & \\\cline{2-5} \hhline{~----~}
 
 \cellcolor{lightYellowTable} & \cellcolor{lightYellowTable} $Y=0$ & \cellcolor{lightGrayTable} $\mathbb{E}[P\:|\:Y=0,X] \bot A$ & $\mbox{Pr}(P_{bin}=1\:|\:Y=0,X) \bot A$ & $\mbox{Pr}(P_{bin}=0\:|\:Y=0,X) \bot A$ & \\
 \cellcolor{lightYellowTable} & \cellcolor{lightYellowTable} & \cellcolor{lightGrayTable} \emph{Balance for Negative Class} & \emph{False Positive Rate} & \emph{True Negative Rate} & \\\cline{2-5} \hhline{~----~}
 
 \cellcolor{lightYellowTable} \multirow{-5}{*}{Separation}& \cellcolor{lightYellowTable} $Y$ & \cellcolor{lightGrayTable} $\mathbb{E}[P\:|\:Y,X] \bot A$ & $\mbox{Pr}(P_{bin}=1\:|\:Y,X) \bot A$ & $\mbox{Pr}(P_{bin}=0\:|\:Y,X) \bot A$ & \\
 
\hhline{------}
 
\cellcolor{lightGreenTable} & \cellcolor{lightGreenTable} $Y=1$ & $\mbox{Pr}(Y=1\:|\:P,X) \bot A$ & \cellcolor{lightGrayTable} $\mbox{Pr}(Y=1\:|\:P_{bin} =1, X) \bot A$ & \cellcolor{lightGrayTable} $\mbox{Pr}(Y=1\:|\:P_{bin} =0, X) \bot A$ & \cellcolor{lightGrayTable} $\mbox{Pr}(Y=1\:|\:P_{bin}, X) \bot A$ \\ 
 \cellcolor{lightGreenTable} & \cellcolor{lightGreenTable} & & \cellcolor{lightGrayTable} \emph{Positive Predictive Value} & \cellcolor{lightGrayTable} \emph{False Omission Rate} & \cellcolor{lightGrayTable} \\ \cline{2-6} \hhline{~-----}

 \cellcolor{lightGreenTable} & \cellcolor{lightGreenTable} $Y=0$ & $\mbox{Pr}(Y=0\:|\:P,X) \bot A$ & \cellcolor{lightGrayTable} $\mbox{Pr}(Y=0\:|\:P_{bin}=1,X) \bot A$ & \cellcolor{lightGrayTable} $\mbox{Pr}(Y=0\:|\:P_{bin}=0,X) \bot A$ & \cellcolor{lightGrayTable} $\mbox{Pr}(Y=0\:|\:P_{bin},X) \bot A$ \\
 \cellcolor{lightGreenTable} \multirow{-4}{*}{Sufficiency} & \cellcolor{lightGreenTable} & & \cellcolor{lightGrayTable} \emph{False Discovery Rate} & \cellcolor{lightGrayTable} \emph{Negative Predictive Value} & \cellcolor{lightGrayTable} \\
 \hline

\end{tabular}
}
\caption{Table of all scan types for CBS for different group fairness definitions. The notation $\bot$ refers to conditional independence from membership in the protected class ($A$). For example, for the False Discovery Rate scan, $\mbox{Pr}(Y=0\:|\:P_{bin}=1,X) \bot A$ is shorthand for $\mbox{Pr}(Y=0\:|\:P_{bin}=1,X, A=1) = \mbox{Pr}(Y=0\:|\:P_{bin}=1,X, A=0) $.\ignore{ Scans that correspond to well-known fairness definitions (e.g., false positive error rate balance) are noted in the table. Shading is used to distinguish between the four main categories of scans (separation for predictions, separation for recommendations, sufficiency for predictions, and sufficiency for recommendations).}}
\label{methods-overview_of_scans_scan_table}
\end{table}
Depending on the particular bias of interest, we can also perform ``value-conditional'' scans by restricting the value of the conditional variable. For example, we scan for subgroups with increased false positive rate (FPR) by restricting the data to individuals with $Y = 0$ and performing a separation scan for recommendations\ignore{ and detecting subgroups with increased proportion of $P_{bin} = 1$}. All of the scan options for CBS are displayed in Table~\ref{methods-overview_of_scans_scan_table}. As we discuss below, each scan displayed in Table~\ref{methods-overview_of_scans_scan_table} can be used to detect bias in either direction, e.g., searching for subgroups with either increased or decreased FPR.

\subsection{Generate Expectations $\hat{I}$ of the Event Variable} \label{methods-generating-hat-p-section}
Once we have defined the event variable $I$ and conditional variable $C$, we wish to detect fairness violations by assessing whether there exist subgroups of the protected class where 
$\mathbb{E}[I \:|\: C, X, A = 1]$ differs systematically from $\mathbb{E}[I \:|\: C, X, A = 0]$.  For each individual $i$ in the protected class, $I_i \:|\: C_i, X_i, A_i = 1$ is observed but $I_i \:|\: C_i, X_i, A_i = 0$ is unobserved. Thus we must calculate an estimate $\hat{I}_i = \mathbb{E}_{H_0}[I_i \:|\: C_i, X_i, A_i = 1]$, under the null hypothesis, $H_0$: $(I \perp A \:|\: C, X$), and compare $\hat{I}_i$ to the observed $I_i$.  To calculate $\hat{I}$ we use the following method from the econometric literature on heterogeneous treatment effects, which controls for non-random selection into the protected class $A$ based on observed covariates $X$: (1) Learn a probabilistic model for estimating $\mbox{Pr}(A=1\:|\:X)$, and use it to produce propensity scores, $p_j^{A}$, for each individual $j$ in the non-protected class; (2) For each individual $j$ in the non-protected class, use the observed $\mathbb{E}[I_j \:|\: C_j, X_j, A_j = 0]$ weighted by the odds of the propensity score for individual $j$, $\frac{p_j^{A}}{1-p_j^{A}}$, to learn a probabilistic model for $\mathbb{E}_{H_0}[I \:|\: C, X, A = 1]$; (3) For each individual $i$ in the protected class, use the model of $\mathbb{E}_{H_0}[I \:|\: C, X, A = 1]$ to calculate $\hat{I}_i = \mathbb{E}_{H_0}[I_i =1 \:|\: C_i, X_i, A_i = 1]$. For a more detailed description of this method, modifications for value-conditional scans, and a discussion of its limitations, please reference Appendix~\ref{generating_i_hat_appendix}.

\subsection{Detect the Most Significant Subgroup $S^{*}$ and Evaluate its Statistical Significance} \label{methods_cbs_score_functions_altss_section}

\begin{table}
\scalebox{.70}{

\begin{tabular}{ |p{1.5cm}|p{2.5cm}|p{3.3cm} p{5.3cm}|p{1.5cm}|p{3.5cm}|}
\hline
 \multicolumn{2}{|c|}{Scan Types}& \multicolumn{2}{|c|}{Hypotheses} & Distribution for $F(S)$ & \multicolumn{1}{|c|}{$F(S)$}\\
 \hline
 \cellcolor{lightYellowTable} & \cellcolor{lightBlueTable} & \cellcolor{lightGrayTable} Null Hypothesis & \cellcolor{lightGrayTable} $H_0: \Delta_i \sim N(0,\sigma)$, $\forall i \in D_1$ & \cellcolor{lightGrayTable} & \cellcolor{lightGrayTable} \\ 
 \cellcolor{lightYellowTable} & \cellcolor{lightBlueTable} & \cellcolor{lightGrayTable} Alternative Hypothesis & \cellcolor{lightGrayTable} $H_1(S): \Delta_i \sim N(\mu,\sigma)$ & \cellcolor{lightGrayTable} & \cellcolor{lightGrayTable} \\ 
 
 \cellcolor{lightYellowTable} & \cellcolor{lightBlueTable} & \cellcolor{lightGrayTable} & \cellcolor{lightGrayTable}where $\Delta_i = \log \left(\frac{I_i}{1- I_i}\right) - \log\left(\frac{\hat{I_i}}{1-\hat{I_i}}\right)$& \cellcolor{lightGrayTable}& \cellcolor{lightGrayTable} \\ 
 \cellcolor{lightYellowTable}& \cellcolor{lightBlueTable} & \cellcolor{lightGrayTable} Over-estimation Bias: & \cellcolor{lightGrayTable} $\mu < 0$, $\forall i \in S$, and $\mu =0$, $\forall i \notin S$. & \cellcolor{lightGrayTable} & \cellcolor{lightGrayTable} \\ 
 \cellcolor{lightYellowTable} & \cellcolor{lightBlueTable} \multirow{-5}{*} {Predictions} & \cellcolor{lightGrayTable} Under-estimation Bias: & \cellcolor{lightGrayTable} $\mu > 0$, $\forall i \in S$, and $\mu =0$, $\forall i \notin S$. & \cellcolor{lightGrayTable} \multirow{-5}{*} {Gaussian} & \cellcolor{lightGrayTable} \multirow{-5}{*} { $\max_{\mu} \frac{2\mu\left(\sum_{i \in S}\Delta_i\right)- |S|\mu^{2}}{2\sigma^{2}}$} \\ \hhline{~-----}

 \cellcolor{lightYellowTable} \multirow{-6}{*}{ Separation } & \cellcolor{lightPeachTable} Recommendations & Null Hypothesis & $H_0: odds(I_i) = \frac{\hat{I_i}}{1-\hat{I_i}}$, $\forall i \in D_{1}$ & \multirow{4}{*}{Bernoulli} & \\ \hhline{--~~~~}
 \cellcolor{lightGreenTable} & \cellcolor{lightBlueTable} 
 Predictions & Alternative Hypothesis & $H_1(S): odds(I_i) = q\frac{\hat{I_i}}{1-\hat{I_i}} $ & & $\max_q\sum_{i \in S}(I_i \log(q)$ \\ \hhline{~-~~~~}
 \cellcolor{lightGreenTable} & \cellcolor{lightPeachTable} & Over-estimation Bias: & $q < 1$, $\forall i \in S$, and $q =1$, $\forall i \notin S$. & & $- \log(q\hat{I_i} - \hat{I_i} + 1))$ \\ 
 \cellcolor{lightGreenTable} \multirow{-3}{*}{Sufficiency} & \cellcolor{lightPeachTable} \multirow{-2}{*}{Recommendations} & Under-estimation Bias: & $q > 1$, $\forall i \in S$, and $q =1$, $\forall i \notin S$. & & \\

\hline
\end{tabular}
}

\caption{Null and alternative hypotheses, $H_0$ and $H_1(S)$, and corresponding log-likelihood ratio score functions, $F(S)$, used to measure a subgroup's degree of anomalousness (comparing the event variable $I$ to its expectation $\hat I$ under $H_0$) for all four variants of CBS. Here, over-estimation (under-estimation) bias means that the expectations $\hat I_i$ are larger (smaller) than $I_i$. \ignore{Shading is used to distinguish between Gaussian and Bernoulli scans.}}
\label{methods-score_function_table}
\end{table}

\ignore{
outline for re-write
- purpose of F(S) -- (1) measure which subgroup is more anomalous (2) fulfull certain properties that allow for efficient scanning
- explanation of the score function
--- denominator
--- numerator - explains u and q (ie: free parameter) that is fit with MLE -- 
----- direction
------leads into explanation of different distributions bc of predictions vs binary outcomes
- penalty
}

\ignore{In Section~\ref{methods-generating-hat-p-section}, we describe the process for generating $\hat{I}$, an estimate of $\mathbb{E}[I \:|\: C, X]$, which is the expectation of event variable $I$ for individuals in the protected class under the null hypothesis of conditional independence, $I \perp A \:|\: (C,X)$.} Given the observed event variables $I_i$ and the expectations $\hat{I}_i$ of the event variable under the null hypothesis ($I \perp A \:|\: C,X$) for the protected class, we define a score function measuring \emph{subgroup bias}, $F:S \to \mathbb{R}_{\geq 0}$, that can be efficiently optimized over exponentially many subgroups to identify $S^\ast = \arg\max_S F(S)$. \ignore{that can be used to compare subgroups of different sizes to identify the subgroup with the most anomalous bias,  $S^\ast = \arg\max_S F(S)$that can be efficiently optimized over the exponentially many subgroups of the data, and can be usefully compared across subgroups of different sizes to identify the ``most significantly biased'' subgroup $S^\ast = \arg\max_S F(S)$.} To do so, we follow the literature on spatial and subset scan statistics~\cite{kulldorff1997spatial,neill2012fast} by defining score functions $F(S)$ that take the general form of a log-likelihood ratio (LLR) test statistic, $F(S) = \log\left(\frac{\Pr(D \:|\: H_1(S))}{\Pr(D \:|\: H_0)}\right)$. Here the denominator represents the likelihood of seeing the observed values of event variable $I$ for subgroup $S$ of the protected class under the null hypothesis $H_0$ of no bias\ignore{, where each observed $I_i$ is drawn from a distribution centered at the corresponding expectation $\hat I_i$}. The numerator represents the likelihood of seeing the observed values of $I$ for subgroup $S$ of the protected class under the alternative hypothesis $H_1(S)$, where the $I_i$ values are systematically increased or decreased as compared to $\hat I_i$. For $H_1(S)$ to represent a deviation from $H_0$, $H_1$ contains a free parameter ($q$ or $\mu$) that is determined by maximum likelihood estimation.  Under-estimation bias ($I_i > \hat I_i$) or over-estimation bias ($I_i < \hat I_i$) can be detected using different constraints for $q$ or $\mu$ as shown in Table~\ref{methods-score_function_table}.\ignore{ finding Whether one wants to detect subgroups of the protected class with either $E[I_i] > \hat I_i$ (under-estimation bias) or $E[I_i] < \hat I_i$ (over-estimation bias) Each scan can be run to detect subgroups of the protected class with either $E[I_i] > \hat I_i$ or $E[I_i] < \hat I_i$, depending on the desired direction of bias to be detected, by restricting $q$ or $\mu$ as shown in Table~\ref{methods-score_function_table}} When $I$ is a probabilistic prediction (i.e., for separation scan for predictions), the hypotheses are in the form of a difference of log-odds between $I$ and $\hat{I}$ sampled from a Gaussian distribution. Here the free parameter $\mu$ in $H_1$ represents a mean shift ($\mu \neq 0$) of the Gaussian distribution.  For all other scans, under $H_0$, each observed $I_i$ is assumed to be drawn from a Bernoulli distribution centered at the corresponding expectation $\hat I_i$. Under $H_1$, the free parameter $q$ represents a multiplicative increase or decrease ($q \neq 1$) of the odds of $I$ as compared to $\hat I$.

As in~\cite{zhang2016identifying}, a penalty term can be added to $F(S)$ equal to a prespecified scalar times the total number of attribute values included in subgroup $S$, summed across all covariates $X^1 \ldots X^m$. Note that there is no penalty for a given attribute if all attribute values are included, since this is equivalent to ignoring the attribute when defining subgroup $S$. This penalty term encourages more interpretable subgroups and reduces overfitting for high-arity attributes\ignore{, which could otherwise result from optimizing the score function over all non-empty subsets of attribute-values}.

We now consider how CBS is able to efficiently maximize $F(S)$ over subgroups $S$ of the protected class, returning $S^\ast = \arg\max_S F(S)$ and the corresponding score $F(S^\ast)$. \ignore{We also note that the statistical significance ($p$-value) of the discovered subgroup $S^\ast$ can be obtained by \emph{randomization testing}, which correctly adjusts for the multiple testing resulting from searching over subgroups. To do so, we can generate a large number of simulated datasets under the null hypothesis $H_0$ of no bias (see the specific formulations of $H_0$ in Table~\ref{methods-score_function_table}), perform the same CBS scan for each null dataset, and compare the maximum score $F(S^\ast)$ for the true dataset to the distribution of maximum scores $F(S^\ast)$ for the simulated datasets. The detected subgroup is significant at level $\alpha$ if its score exceeds the $1-\alpha$ quantile of the $F(S^\ast)$ values for the simulated datasets.} The scan procedure for CBS takes as inputs a dataset $D_1 = (I, \hat I, X)$ consisting of the event variable $I_i$, the estimated expectation of $I_i$ under the null hypothesis $\hat I_i$, and the covariates $X_i$, for each individual in the protected class ($A_i=1$), along with several parameters: the type of scan (Gaussian or Bernoulli), the direction of bias to scan for (over- or under-estimation), complexity penalty, and number of iterations. It then searches for the highest-scoring subgroup (consisting of a non-empty subset of values $V^j$ for each covariate $X^j$), starting with a random initialization on each iteration, and proceeding by \emph{coordinate ascent}. The coordinate ascent step identifies the highest-scoring non-empty subset of values $V^j$ for a given covariate $X^j$, conditioned on the current subsets of values $V^{-j}$ for all other attributes. As shown in~\cite{mcfowland2018efficient}, each individual coordinate ascent step can provably find the optimal subset of attribute values while evaluating only $|X^j|$ of the $2^{|X^j|}$ subsets of values, where $|X^j|$ is the arity of covariate $X^j$. This efficient subroutine follows from the fact that the score functions above satisfy the additive linear-time subset scanning property~\cite{neill2012fast,speakman2016penalized}. The coordinate ascent step is repeated with different, randomly selected covariates until convergence to a local optimum of the score function, and multiple random restarts enable the scan to approach the global optimum~\cite{mcfowland2018efficient}.
For an in-depth, self-contained description of the scan algorithm, including pseudocode, and how it exploits an additive property of the score functions to achieve linear-time efficiency for each scan step, see Appendix~\ref{fss_for_cbs_appendix}.  

Finally, as described in detail in Appendix~\ref{permutation_testing}, we perform \emph{permutation testing} to compute the p-value of the detected subgroup, comparing its score to the distribution of maximum subgroup scores under $H_0$, and report whether it is significant at a given level $\alpha$ (e.g., $\alpha=.05$).

\section{Evaluation} \label{experiments_validation}

We evaluate the CBS framework through semi-synthetic simulations.\ignore{Rather than using the COMPAS risk predictions, we used the covariates from the COMPAS data and generated synthetic outcomes, predictions, and recommendations,} We generate 100 semi-synthetic datasets using the COMPAS data, described in Section~\ref{compas_experiment_section}, where for each dataset, we randomly select an attribute and value to define the protected class $A$, remove that attribute from $X$, and randomly select a subgroup of the protected class $S_{bias}$ (defined by a non-empty subset of values for each attribute $X^1 \ldots X^m$) into which we will inject biases or base rate shifts.  We pick $S_{bias}$ by randomly choosing two attributes ($n_{bias} = 2$) and then independently including or excluding each value of those attributes with probability $p_{bias}$ of being included in $S_{bias}$, with $p_{bias}=0.5$.  For each attribute-value of the covariates, we draw a weight from a Gaussian distribution, $\mathcal{N}(0,0.2)$. We use these weights to produce the true log-odds of a positive outcome ($Y_i=1$) for each individual $i$ by a linear combination of the attribute values with these weights. Additionally, for each individual, we add $\epsilon_{i}^{true} \sim \mathcal{N}(0,\sigma_{true})$ to their true log-odds, representing variation between individuals that arises from factors other than the set of scan attributes, and is incorporated into the predictive model.\footnote{Rudin et al.~\cite{rudin2020age} state that COMPAS relies on up to 137 variables collected from a questionnaire. Some of these additional predictors may be informative, but as they note, this complexity creates numerous additional problems ranging from lack of transparency to adverse impacts from data entry errors.} 
Given the true log-odds of $Y_i = 1$ for each individual, we draw each outcome $Y_i$ from a Bernoulli distribution with the corresponding probability. Next, we set each individual's predicted log-odds equal to their true log-odds plus $\epsilon_i^{predict} \sim \mathcal{N}(0,\sigma_{predict})$, representing non-systematic errors (random noise) in the predictive model. We use default values of $\sigma_{true} = 0.6$ and $\sigma_{predict} = 0.2$, but examine the robustness of our results to these parameter values in Appendix~\ref{additional_validation_setups}; see also Appendix~\ref{additive_term_validation} for discussion of the impact of $\sigma_{true}$ on sufficiency-based definitions of fairness. Finally, we threshold the probabilities to produce binarized recommendations $P_{i,bin} = \mathbbm{1}{(P_{i} \ge 0.5)}$ for each individual $i$.  Using covariates from the COMPAS data with synthetic outcomes, predictions, and recommendations enables us to quantify and compare the detection performance of CBS and competing methods for known biases injected into the data. 

We compare the four variants of CBS to GerryFair~\citep{kearns2018preventing}, a framework for detecting biases and learning classifiers that guarantee subgroup fairness, and Multiaccuracy Boost~\cite{kim2019multiaccuracy}, a method for auditing subgroup inaccuracies and post-processing classifiers to improve performance to ensure subgroup accuracy. For more information about the modifications we made to both methods to make them more comparable to CBS for these simulations, see Appendix~\ref{benchmark_methods_adaptions}.
We use the same settings for CBS as described in Section~\ref{compas_experiment_section}, with the exception of running all scans with all conditional variable values rather than as value-conditional scans.
We designed the evaluation to answer three questions: 
\begin{enumerate} [label=(Q\arabic{*}),noitemsep,topsep=0pt,parsep=0pt,partopsep=0pt,leftmargin=*]
 \item \label{Q1} How well do the four variants of CBS (and competing methods) detect \emph{biases} injected into subgroup $S_{bias}$ of the protected class, in the form of a systematic difference between the predicted and true log-odds of the event variable $I$ for $S_{bias}$?\ignore{, in which the predicted log-odds within some subgroup for the protected class (and thus, the binarized recommendations based on those predictions) differ systematically from the true log-odds of a positive outcome?}
 \item \label{Q2} How do the four variants of CBS (and competing methods) respond to a \emph{base rate shift} for subgroup $S_{bias}$ of the protected class, in the form of concurrently shifting both the predicted and true log-odds of the event variable $I$ for $S_{bias}$, assuming no injected bias? \ignore{, in which the true log-odds within some subgroup for the protected class are systematically higher or lower than the true log-odds for the corresponding subgroup of the non-protected class, assuming no injected bias (i.e., the predicted log-odds do not differ systematically from the true log-odds)?}
 \item \label{Q3} To what extent do answers for \ref{Q1} and \ref{Q2} vary depending on the characteristics of $S_{bias}$?
\end{enumerate}

To address \ref{Q1}, we inject bias into subgroup $S_{bias}$ of the protected class, keeping the corresponding subgroup of the non-protected class unchanged, in one of two ways: 
(1) increasing the predicted log-odds by $\mu_{sep}$ for each individual in $S_{bias}$, and recomputing the model's predicted probabilities $P_i$ and recommendations $P_{i,bin}$; or
(2) reducing the true log-odds by $\mu_{suf}$ for each individual in $S_{bias}$, and redrawing the outcomes $Y_i$. Both of these shifts result in a bias in which the predicted values ($P$ or $P_{bin}$) overestimate the outcomes ($Y$) for the given subgroup of the protected class. We distinguish between these two biases because $\mu_{sep}$ corresponds to a shift of the predicted log-odds (the alternative hypothesis for separation scans) and $\mu_{suf}$ corresponds to a shift of the outcomes (the alternative hypothesis for sufficiency scans). 
To address \ref{Q2}, we inject a base rate shift into subgroup $S_{bias}$ of the protected class, keeping the corresponding subgroup of the non-protected class unchanged. This is done by increasing \emph{both} the true log-odds and the predicted log-odds by $\delta$, then redrawing outcomes $Y_i$ and recomputing predictions $P_i$ and recommendations $P_{i,bin}$. For positive $\delta$, this creates a higher base rate of a positive outcome for subgroup $S_{bias}$ of the protected class, as compared to the corresponding subgroup of the non-protected class, while maintaining well-calibrated predictions. For \ref{Q3}, we vary the size of $S_{bias}$ in two ways. First, we vary the number of attributes, $n_{bias}$, that the attribute-values can be chosen from, between 1 and 4. Second, we vary the probability, $p_{bias}$, that each value of the chosen attributes is included in $S_{bias}$. We run three experiments ($\mu_{sep} =1$, $\mu_{suf} =1$, and $\delta = 0.5$) while varying $n_{bias}$ and $p_{bias}$.

Once we inject bias into or shift the base rates of $S_{bias}$ in the protected class, we run all our scans and GerryFair and Multiaccuracy Boost on our data. We measure the accuracy of the detected subset, which we refer to as $S^\ast$, for all the scans and methods with the following accuracy measurement: $\frac{\:|\:S_{bias} \cap S^\ast\:|\:}{\:|\:S_{bias} \cup S^\ast\:|\:}$, or Jaccard similarity between the injected and detected subsets. Accuracies are averaged over 100 semi-synthetic datasets for each experiment.


\begin{figure}
 \centering
 \includegraphics[scale=.70]{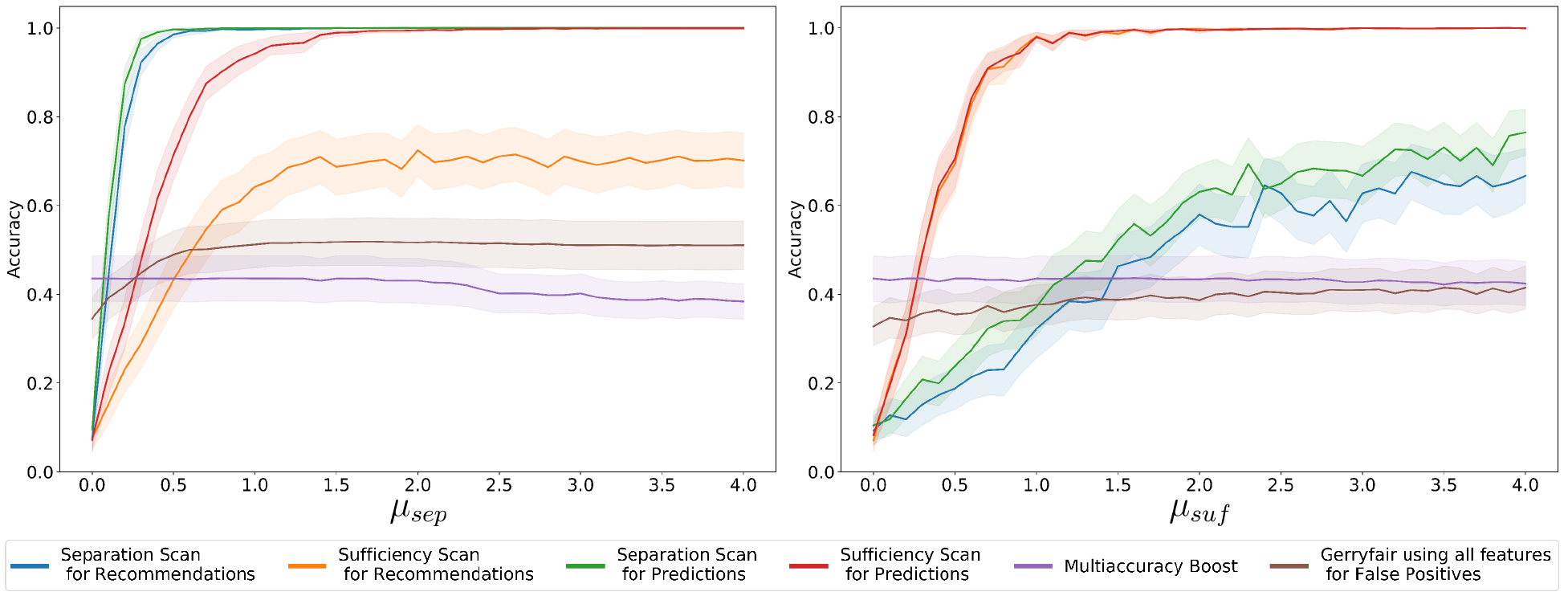}
 \caption{Average accuracy (with 95\% CI) as a function of the amount of bias injected into subgroup $S_{bias}$ of the protected class, for four variants of CBS, GerryFair, and Multiaccuracy Boost. Left: increasing predicted log-odds by $\mu_{sep}$. Right: decreasing true log-odds by $\mu_{suf}$.}
 \label{a_diagrams}
\end{figure}

\begin{figure}
 \centering
 \includegraphics[scale=.4]{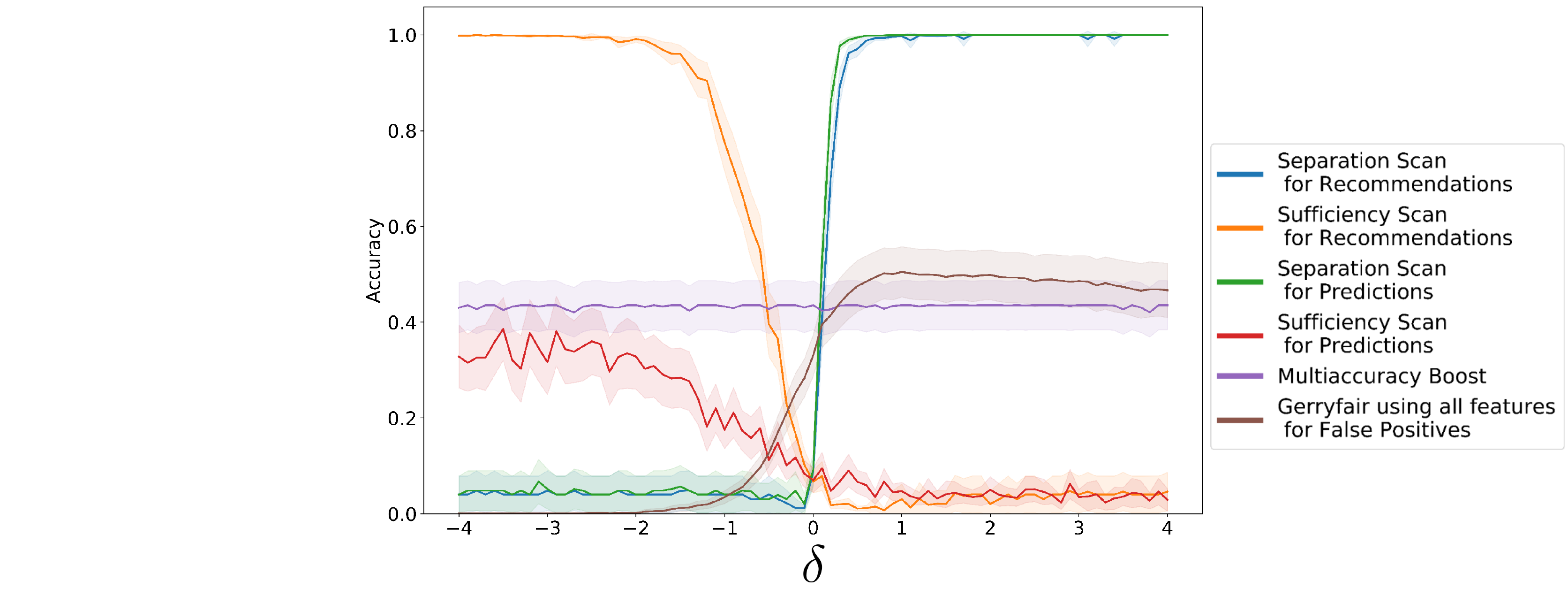}
 \caption{Average accuracy (with 95\% CI) as a function of the base rate difference $\delta$ between protected and non-protected class for subgroup $S_{bias}$, for four variants of CBS, GerryFair, and Multiaccuracy Boost. Note that predictions are well calibrated, $\mu_{sep}=\mu_{suf}=0$.}
 \label{2a_diagrams}
\end{figure}

In Figure~\ref{a_diagrams}, which addresses \ref{Q1}, we observe that all four variants of CBS are able to detect the injected bias (for subgroup $S_{bias}$ of the protected class) with higher accuracy than GerryFair or Multiaccuracy Boost. Sufficiency scans had highest accuracy for a shift in true log-odds ($\mu_{suf}$), and separation scans had highest accuracy for a shift in predicted log-odds ($\mu_{sep}$). Scans for predictions generally outperformed scans for recommendations, due to the loss of information from binarizing the probabilistic predictions. Interestingly, sufficiency scan for predictions (but not for recommendations) converged to perfect accuracy for $\mu_{sep}$, while separation scans did not converge to perfect accuracy for $\mu_{suf}$. Sufficiency scan for predictions is conditioned on a real-valued variable ($P_i$) rather than a binary variable ($P_{i,bin}$ or $Y_i$), allowing more flexible modeling of $\mathbb{E}[Y \:|\: P, X]$ and thus greater sensitivity to shifts in predicted log-odds. 

In Figure~\ref{2a_diagrams}, which addresses \ref{Q2}, shifting the base rate for subgroup $S_{bias}$ of the protected class results in separation scans detecting a base rate shift when $\delta > 0$, while sufficiency scans and competing methods are robust to this shift.  This finding aligns with previous research proving that differences in base rates between two populations will result in a higher false positive rate for the population with a higher base rate when using a well-calibrated classifier~\citep{chouldechova2017fair}. \ignore{As noted in~\cite{corbett2018measure}, ``higher FPR for one group relative to another may either mean that the group faces a
lower threshold, indicating discrimination, or alternatively, that the group has a higher base rate'', and this effect of base rate increases appears to be true for separation-based definitions of fairness more generally. }
Interestingly, sufficiency scan for recommendations detects a base rate shift for $\delta \ll 0$. In this case, $\mathbb{E}[Y \:|\: P_{bin}, X]$ is lower for individuals in the protected class than for individuals with negative recommendations in the non-protected class. Thus conditioning on the binary indicator $P_{i,bin}$ is not sufficient to capture this decrease in the true log-odds, while conditioning on the real-valued prediction $P_i$ allows sufficiency scan for predictions to extrapolate reasonably well to these cases. 

In Figure~\ref{cd_diagrams} in Appendix~\ref{varying_evaluation}, which addresses \ref{Q3}, we see that increasing the number of affected dimensions $n_{bias}$ generally decreases performance, with the relative accuracies for scans and competing methods similar to those in Figures~\ref{a_diagrams} and \ref{2a_diagrams}. Interestingly, increasing $p_{bias}$ to 1 (meaning that bias is injected into the entire protected class) enables GerryFair (for $\mu_{sep}$) and Multiaccuracy Boost (for $\mu_{sep}$ and $\mu_{suf}$) to achieve similar accuracy to CBS, but both methods underperform for smaller, more subtle, subgroup biases.  Additional robustness checks are described in Appendix~\ref{additional_validation_setups}, and for estimates of compute power needed for the simulations see Appendix~\ref{simulations_compute_power}.\ignore{
For all 100 semi-synthetic datasets, CBS was run 150,400 times for the evaluation simulations and robust checks in Appendix~\ref{additional_validation_setups}. We used 15 shared, university compute servers running CentOS with 16-64 cores, and 16 - 256 GB of memory.  Each server performed 15-120 runs of CBS concurrently, and ran for approximately 9 days. }
\ignore{
Finally, we performed experiments to test the robustness of CBS for varying parameters ($\sigma_{predict}$, $\sigma_{true}$), as well as exploring how base rate shifts impact detection of injected bias. These experiments are described in Appendix~\ref{additional_validation_setups}. }
\section{Case Study of COMPAS} \label{compas_experiment_section}

The COMPAS algorithm is used in various jurisdictions across the United States as a decision support tool to predict individuals' risk of recidivism. It is commonly used by judges when deciding whether an arrested individual should be released prior to their trial~\citep{angwin2016machine}. \ignore{Following the initial investigation by ProPublica about fairness issues in COMPAS risk predictions~\citep{angwin2016machine}, ProPublica's COMPAS dataset has been used as a benchmark in the fairness literature. While we use the COMPAS data because of its familiarity and supporting research, we also note the value of alternative framings of the evaluation of automated decision support tools in the criminal justice systems, examining the risks that the system poses to defendants rather than the risk of the defendants to public safety~\cite{mitchell2021algorithmic,meyer2022flipping}.}\ignore{ We follow many of the processing decisions made in the initial ProPublica analysis, including removing traffic offenses and defining recidivism as a new arrest within two years of the initial arrest for a defendant~\citep{larson_angwin_kirchner_mattu_2016, larson_roswell_2017}. After preprocessing the initial data set, we have 6,172 defendants, their gender, race, age (Under 25 or 25+), charge degree (misdemeanor or felony), prior offenses (None, 1 to 5 or Over 5), predicted recidivism risk score (1-10), and whether they were re-arrested within two years of the initial arrest.} We define each defendant's predicted probability of reoffending, $P_i$, by mapping their COMPAS risk score to the proportion of all defendants with the given risk score who reoffended. Defendants with COMPAS risk scores of 5+ are considered ``high risk'' ($P_{i,bin}=1$) since COMPAS instructs that defendants in this score range should be considered carefully by supervision agencies ~\citep{larson_angwin_kirchner_mattu_2016}.  For more details about the COMPAS data, as well as critiques of this dataset, please see Appendix~\ref{compas_preprocessing} and Appendix~\ref{compas_considerations} respectively.

We chose the parameters for each of the four variants of CBS (value of the conditioning variable, if it is binary, and direction of effect) in order to search for systematic biases in the COMPAS predictions and recommendations which disadvantage the protected class. For the separation scans, we detect positive deviations for the protected class attribute in the $\mathbb{E}(P\:|\:Y=0,X)$ and $\mbox{Pr}(P_{bin}=1\:|\:Y=0,X)$, i.e., increase in predicted risk and increase in FPR for non-reoffending defendants, respectively. For the sufficiency scans, we detect a negative deviation for the protected class in the $\mbox{Pr}(Y=1\:|\:P,X)$ and $\mbox{Pr}(Y=1\:|\:P_{bin}=1,X)$, i.e., decreased probability of reoffending conditional on predicted risk and on being flagged as high-risk, respectively. For a discussion on the considerations for using COMPAS in this case study and for choosing different group fairness definitions for pre-trial tools, please reference Appendix~\ref{compas_considerations}.
\ignore{Thus, for the separation scans, we search for subgroups of the protected class with the most significant \emph{increase}, either in the probabilistic predictions or in the probability that the binarized recommendation equals 1, conditional on the defendant's covariates. Moreover, we perform value-conditional scans, focusing specifically on the subset of defendants who did not reoffend ($Y_i = 0$), since these are the individuals who would be harmed the most by an overestimate of their reoffending risk prediction or by a ``false positive'' recommendation to treat them as ``high risk''. For the sufficiency scans, we search for subgroups of the protected class with the most significant \emph{decrease} in the observed rate of reoffending, conditional on the defendant's covariates and their COMPAS prediction or recommendation. For the sufficiency scan for recommendations, we also perform a value-conditional scan. We focus specifically on the subset of defendants who were predicted to be ``high risk'' by COMPAS ($P_{i,bin}=1$) because this labeling could negatively impact the defendant, e.g., by decreasing their likelihood of pre-trial release. }\ignore{This results in CBS detecting subgroups of the protected class for whom the positive predictive value is most significantly decreased, or equivalently, for whom the false discovery rate is most significantly increased.}For all scans, we use all attributes except for the sensitive attribute when calculating the probability of being a member of the protected class (for the propensity score weighting step) and when generating the predicted values $\hat{I}$ in Section~\ref{methods-generating-hat-p-section}. All scans were run for 500 iterations with a penalty equal to 1.

\begin{figure}
 \centering
 \includegraphics[width = .9\textwidth]{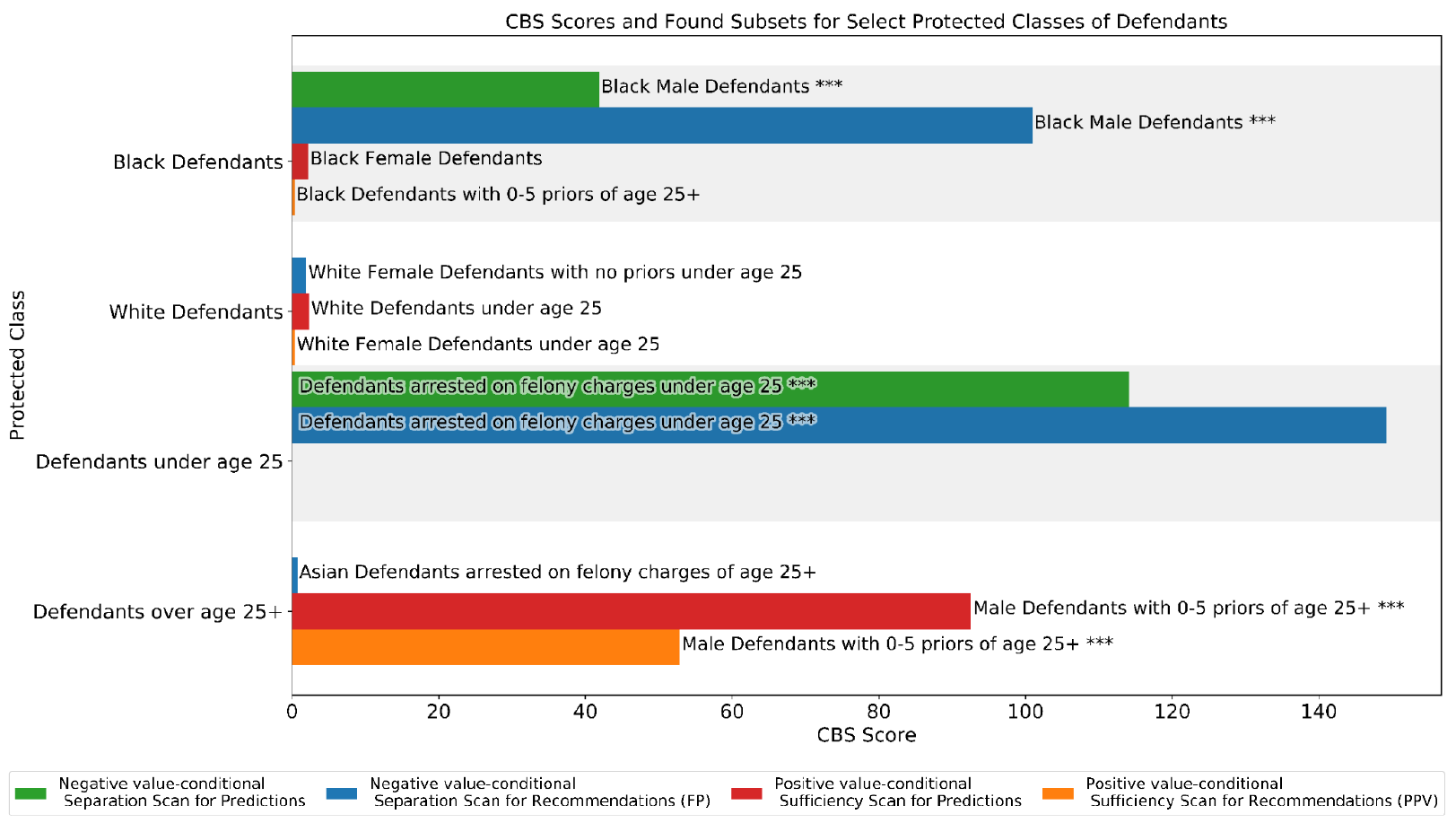}
 \caption{Scores of the subgroups found when running four variants of CBS on COMPAS data for different choices of protected class. A text description of the subgroup $S^\ast$ found for each scan is provided if the subgroup score $F(S^\ast)$ is greater than 0.  *** indicates the subgroup's score is statistically significant with p-value < .05 measured by permutation testing described in Appendix~\ref{permutation_testing}.}
 \label{complas_bar_plot}
\end{figure}

Figure~\ref{complas_bar_plot} contains the detected subgroups $S^\ast$, and their associated log-likelihood ratio scores $F(S^\ast$) and corresponding indicators of statistical significance, found by each of the four variants of CBS, for various choices of the protected class: Black, white,\ignore{ female, male,} younger (under the age of 25) and older (age 25+) defendants.  Please see Appendix~\ref{permutation_testing} for the  permutation test procedure used to determine statistical significance of CBS's detected biases. For the full set of results for all CBS scans when treating each attribute value as the protected class, please see Table~\ref{compas_full_results} in Appendix~\ref{compas_full_results_appendix}. This table also includes information about the number of individuals and the observed rate (e.g., proportion of reoffending), both for the detected subgroup of the protected class, and for the corresponding (comparison) subgroup of the non-protected class. 

\ignore{Overall, we observe from Figure~\ref{complas_bar_plot} and Table~\ref{compas_full_results} that the separation scans highlighted biases (high FPR and overestimated predicted probabilities of reoffending, among defendants who did not reoffend) against subgroups of three protected classes--younger, Black\ignore{, and female} defendants. In contrast, the sufficiency scans highlighted biases (high false discovery rates, among defendants who were predicted as ``high risk'', and low reoffending rates, controlling for predicted risk) against older defendants\ignore{ and different subgroups of female defendants}.} Below we discuss statistically significant racial and age biases that CBS found in COMPAS's predictions and recommendations, comparing subgroups found by separation and sufficiency scans:

\textbf{Racial bias in COMPAS.} Figure~\ref{complas_bar_plot} shows that the separation scans identify highly significant biases against a subgroup of Black defendants, while the sufficiency scans do not. These results support and complement the previous findings by ProPublica~\cite{angwin2016machine} and follow-up analyses~\cite{chouldechova2017fair}, which concluded that COMPAS has large error rate disparities which negatively impact Black defendants (corresponding to large scores for separation scans), and that its predictions are well-calibrated for Black defendants (corresponding to small scores for sufficiency scans)\ignore{, and that the two fairness definitions are incompatible given differences in the base rate of recorded reoffending/rearrest between Black and white defendants~\cite{kleinberg2016inherent,chouldechova2017fair}}. However, CBS's detected subgroup for the two separation scans adds a useful finding to this discussion: the large FPR disparity of COMPAS against Black defendants is even more significant in the intersectional subgroup of Black males. Non-reoffending Black male defendants have an FPR of 0.44, compared to non-reoffending non-Black male defendants' FPR of 0.19, whereas non-reoffending Black defendants have an FPR of 0.42, compared to non-reoffending non-Black defendants' FPR of 0.20.\ignore{ Finally, CBS does not find any high-scoring subgroups for white defendants, suggesting that COMPAS predictions do not disadvantage subgroups of white defendants. CBS also does not find high-scoring subgroups for other racial groups, though this could be due to a small sample size rather than unbiased predictions.}

\textbf{Age bias in COMPAS.} Previous research on COMPAS has argued that it relies heavily on age, and specifically, the assumption that younger defendants are more likely to reoffend~\citep{rudin2020age}, when computing risk scores. Younger defendants have a higher reoffending rate compared to older defendants (0.56 vs.~0.46), and thus, well-calibrated predictions and recommendations would result in younger defendants having higher FPR than older defendants. Our separation scans identify non-reoffending defendants under age 25 arrested on felony charges as the subgroup with the largest FPR disparity.\ignore{: these defendants have a 55\% FPR and average predicted probability of reoffending of 52\%, as compared to non-reoffending defendants of age 25+, who have a 29\% FPR and average predicted probability of reoffending of 39\%.} 
On the other hand, our sufficiency scans identify a large subgroup bias within the protected class of older (age 25+) defendants: older male defendants with 0-5 priors have a lower rate of reoffending, as compared to younger male defendants with 0-5 priors, both for flagged high-risk defendants (sufficiency for recommendations) and for defendants with similar risk scores (sufficiency for predictions). This finding highlights the scenario described in Section~\ref{sec:intro} that CBS is designed to detect: predictions are well-calibrated between older and younger defendants, in aggregate, but not for the detected subgroup of older males with 0-5 priors.

\section{Limitations} \label{discussion_and_limitations}

Our CBS framework is designed to audit a classifier's predictions and recommendations for biases with respect to subgroups of a protected class, whereas competing methods provide mechanisms for both auditing and correcting classifiers. Combining auditors with correction and training presents two challenges, one being how to quantify the inherent trade-offs between performance and fairness when correcting for subgroup biases. Another, more subtle implication of designing auditors that are linked to correction and training methods is that it reinforces the framing that the primary solution to subgroup biases is to correct the models. Given that fairness is often context-specific, ideas of fairness could differ between stakeholders, and upstream biases exist in data sources used in many socio-technical settings, designing an optimally fair model is not always feasible.\ignore{ Feller et al. mention, with respect to discovering biases in COMPAS risk scores, that "rather than using risk scores to determine which defendants must pay money bail, jurisdictions might considering ending bail requirements altogether \ldots so that no one is unnecessarily jailed"~\citep{feller2016computer}. While the broader question of how to address biases in criminal justice decision tools is outside of the scope of our work,} We endorse exploring larger policy shifts to address biases that auditing tools like CBS might unearth that are correlated with broader societal issues, as well as exploring different ways (not limited to model correction) to address these biases.

CBS is designed to detect biases that take the form of group fairness violations represented as conditional independence relationships.  While CBS is easily generalizable to other objectives that can be represented as group-level conditional independence relationships, it is less generalizable to other fairness definitions such as individual and counterfactual fairness definitions~\citep{dwork2012fairness, kusner2017counterfactual}.  Our technique for estimating the expectations $\hat{I}$ for individuals under the null hypothesis of no bias has the limitation (which is commonly cited in the average treatment effects literature) of only being reliable when using well-specified models for estimating the propensity scores of protected class membership and for estimating $\hat{I}$.  Given the consistency of our COMPAS results in Section~\ref{compas_experiment_section} with other researchers' findings about COMPAS, the process of estimating $\hat{I}$ seems to model the COMPAS data well.  With that said, we encourage users of CBS to check estimates of $\hat{I}$ and if necessary, employ procedures common in the econometric literature (such as~\citep{imbens2004nonparametric,schuler2017targeted}) or calibration methods within the computer science literature.  Lastly, there are various limitations to permutation testing, some of which are discussed in~\cite{berger2000pros}, such as the conservatism of p-values obtained by permutation tests.  For CBS specifically, if $\hat{I}$ is poorly estimated during permutation testing, this could result in higher type II errors where CBS is more likely to erroneously accept the null hypothesis $H_0$ of no bias.

Our simulation experiments in Section~\ref{experiments_validation} account for bias in the form of shifts in the predicted and true log-odds (separately and jointly) -- which produces predictive and aggregation biases -- for a prescribed set of covariate attribute values in the protected class. In real-world scenarios, the generative process of bias might differ from the assumptions made in our simulations in various ways including: (1) noisy biased subgroups in the protected class not easily defined by discrete sets of covariate attribute values; (2) irregular shifts in the predicted and true log-odds within a biased subgroup that are not related to covariate attributes; (3) other forms of predictive bias not defined by conditional independence statements (i.e., outside of separation and sufficiency group fairness metrics).  Future research could determine and (if necessary) improve CBS's robustness to different generative schemas of bias. \ignore{ there could be multiple kinds of bias that could exist in various forms that are not explicitly detected by CBS. }
\ignore{, such as bias in the form of missing or inaccurate covariates or outcomes for a subgroup of the protected class. Future research in methods that are robust to co-occurring forms of bias would be beneficial.} 

In summary, CBS is a framework that works with most group-level fairness definitions to detect intersectional and contextual biases within subgroups of the protected class. CBS\ignore{ allows for the flexible utilization of group-fairness definitions but} overcomes some of the issues that arise when only considering fairness violations in aggregate for singular protected attribute values. We showed that the CBS framework can discover previously undocumented intersectional and contextual biases in COMPAS scores, and that it outperforms similar methods that audit classifiers for subgroup fairness. Despite the limitations above, we hope that this work will be useful to practitioners wishing to identify and correct more subtle subgroup biases in decision-support tools.

\begin{ack}

This material is based upon work supported by the National Science Foundation Program on Fairness in Artificial Intelligence in Collaboration with Amazon, grant IIS-2040898. Any opinions, findings, and conclusions or recommendations expressed in this material are those of the authors and do not necessarily reflect the views of the National Science Foundation or Amazon.

\end{ack}

\bibliographystyle{ACM-Reference-Format}
\bibliography{arxiv_submission_6_21_2023}
\clearpage
\appendix

\ignore{

APPENDICES OUTLINE 

A Methods appendices
A.1 further details and limitations for methods used to generate $\hat$ I 
A.2 FSS for conditional bias scna
-A.2.1 formal definition of ALTSS
-A.2.2 Pseudocode of Scan Algorithm
A.3 Permutation testing
A.4 All conditional bias scan parameters 

B Evalutation Appendices
B.1 Adaption of Competing methods for benchmarks
B.2 explanation of sigma_true
B.3 Evaluation of n_{bias} and p_{bias}
B.4 Additional Evaluation Simulations
B.5 estimate of compute power

C Case Study of COMPAS Appendices
C.1 Additional information about Preprocessing of COMPAS Data
C.2 Full results of COMPAS case study
C.3 Consideration and Limitations of COMPAS Usage
}

\section{Methods Appendices}

\subsection{Details about the Method for Generating $\hat{I}$ used in Section~\ref{methods-generating-hat-p-section} and its Limitations} \label{generating_i_hat_appendix}

The method presented in Section~\ref{methods-generating-hat-p-section} describes how to estimate $\hat{I_i}$, the expectation of the event variable $I_i$ for each individual $i$ in the protected class, under the null hypothesis, $H_0$, of no bias (i.e., $I \perp A \:|\: C, X$).  Using the estimated $\hat{I}$ and observed $I$, we can determine which subgroups in the protected class have the largest deviations in $I$ as compared to what we would expect if there was no bias, $\hat{I}$.  The method to generate $\hat{I}$ borrows from the literature on causal inference in observational settings, where propensity score reweighting is used to account for the selection of individuals into a ``treatment'' condition (here, membership in the protected class) given their observed covariates $X$. 

The method to estimate $\hat{I}$ consists of the following steps:
\begin{enumerate}
  \item \label{prop_score_model} Train a predictive model using all the individuals in the data to estimate $\mbox{Pr}(A =1 \:|\: X)$.
  \item Use this model to produce the probabilities, $p_i^A = \mbox{Pr}(A_i=1 \:|\: X_i)$, and the corresponding propensity score weights, $w_i^A = \frac{p_i^A}{1-p_i^A}$, for each individual $i$ in the non-protected class ($A_i=0$). Intuitively, individuals in the non-protected class whose attributes $X_i$ are more similar to individuals in the protected class have higher weights $w_i^A$.
  This weighting scheme is used in the literature to produce causal effect estimates that can be interpreted as the average treatment effect on treated individuals (ATT) under typical assumptions of positivity and strong ignorability.
  \item \label{H_0_step}If the event variable, $I$, is binary (i.e., for all sufficiency scans and separation scan for recommendations), we train a model using only data for individuals in the non-protected class ($A_i=0$) to estimate $\mathbb{E}_{H_0}[I \:|\: C, X]$ by weighting each individual $i$ in the non-protected class by $w_i^A$.\ignore{ We use the weights $w_i^A$ to weight the data used to train the model during the fitting process.}  The trained model is used to estimate the expectations $\hat I_i = \mathbb{E}_{H_0}[I_i \:|\: C_i, X_i]$ for each individual in the protected class ($A_i =1$) under the null hypothesis, $H_0$, of $I \perp A \:|\: (C, X)$.

  \item \label{H_0_step_alt} For the separation scan for predictions, we have a real-valued event variable, the probabilistic predictions $P$,  rather than a binary event variable. We use a similar but modified process to estimate $\mathbb{E}_{H_0}[I\:|\:C, X]$, where $I=P$ and $C=Y$. For each individual $i$ in the non-protected class, we create two training records containing the same covariates $X_i$, but different labels and associated weights:
  \begin{enumerate}
      \item For the first record, we set the label, $I_{i_{+}}^{temp}$, equal to 1, and set the weight to $w_{i}^A P_i$.
      \item For the second record, we set the label, $I_{i_{-}}^{temp}$, equal to 0, and set the weight to $w_{i}^A(1- P_i)$
  \end{enumerate} 
  We create a dataset that includes both records for each individual in the non-protected class and their associated weights,\ignore{We concatenate both copies of the variables, their event variables and modified weights together} and use this concatenated data set to train a model that estimates $\mathbb{E}_{H_0}[I^{temp}\:|\:C, X]$, by weighting each individual $i$ in the non-protected class by either $w_{i}^A P_i$ or  $w_{i}^A(1- P_i)$ as described above. This approach is consistent with other CBS variants and enforces the desired constraint $0 \le \hat I_i \le 1$, unlike alternative approaches such as using regression models to predict $P$.\ignore{ this approach is consistent with other CBS variants and enforces the desired constraint $0 \le \hat I_i \le 1$.}
  \end{enumerate}
  For value-conditional scans, CBS audits for biases in the subset of data where $C=z$, for $z \in \{0,1\}$.  Dataset $D$ is filtered before Step~\ref{H_0_step} to only include individuals where $C=z$. For example, for the value-conditional scan for FPR, we filter the data to only include individuals where $C=0$ (or equivalently, $Y=0$). \ignore{Lastly, for scans where we are auditing on a negative event, such as the true negative rate or false discovery rate scans, the steps are identical except we train our model to estimate $\mbox{Pr}(I=0\:|\:X', C=z)$, and then $\hat{P} = \{p_i = \mbox{Pr}(I_i = 0 \:|\: x_i', C_i) \forall i \:|\: r_{i_k} =1, C_i = z\}$.}
  
A probabilistic model can be used to estimate $\mbox{Pr}(A =1 \:|\: X)$ in Step~\ref{prop_score_model}, and a probabilistic model that allows for weighting of instances during training can be used to estimate $\mathbb{E}_{H_0}[I \:|\: C, X]$ in Steps~\ref{H_0_step} and~\ref{H_0_step_alt}. For Sections~\ref{experiments_validation} and \ref{compas_experiment_section}, as well as Appendices~\ref{varying_evaluation} and \ref{additional_validation_setups},  we use logistic regression to estimate $\mbox{Pr}(A =1 \:|\: X)$  and weighted logistic regression to estimate $\mathbb{E}_{H_0}[I \:|\: C, X]$. When estimating $\mathbb{E}_{H_0}[Y \:|\: P, X]$ (the realized expectation of $\mathbb{E}_{H_0}[I \:|\: C, X]$) for sufficiency scan for predictions, we transform the conditional variable, $P_i$, to its corresponding log-odds, $\log \frac{P_i}{1-P_i}$, prior to training, since we expect $\log \frac{Y_i}{1-Y_i}$ (the target of the logistic regression) to be approximately $\log \frac{P_i}{1-P_i}$ for well-calibrated classifiers.
Alternative prediction models, such as random forests with Platt scaling for calibration of probability estimates, could also be used in place of logistic regression.

The method described above has the limitation of only producing accurate estimates of $\hat{I}$ when both the model for $\mbox{Pr}(A =1 \:|\: X)$ and $\mathbb{E}_{H_0}[I \:|\: C, X]$ are well-specified.  Accurate estimates of $\hat{I}$ are essential for CBS to accurately detect the subgroup in the protected class with the most deviation between the observed $I$ and estimated $\hat{I}$ under the null hypothesis of no bias.  Given the consistency of our findings for the COMPAS case study in Section~\ref{compas_experiment_section} with other researchers' findings about COMPAS, as well as other checks we have performed to examine $\hat{I}$, we believe the method above suffices for our research purposes.  We encourage others using CBS to be aware of this limitation, pay special consideration to estimates of $\hat{I}$, and if necessary, employ methods from the causal inference literature on doubly robust estimation (such as  \citep{imbens2004nonparametric,schuler2017targeted}) or methods from the computer science literature for model calibration when producing estimates of $\hat{I}$. 
\ignore{Additionally, simulations similar to those described in Section~\ref{experiments_validation} could be used as a cross validation process for finding models and hyperparameters to estimate $\hat{P}$.
The method above is one approach for producing $\hat{P}$ that follows the steps often used in statistical analysis literature to produce inverse propensity scores that are later used to calculate average treatment effects. There have been advances in this field, specifically for addressing issues of robustness of estimates when models are mis-specified. One example of a mis-specified model in our outlined method above would be if some of the $p_{i}^{r}$'s poorly estimate the propensity of protected class membership and therefore some of the $W^r$'s are too high or too low and resultantly the model estimating $\hat{P}$ is poorly fit as well. Some of the research that addresses these issues of mis-specification and robustness include normalization methods for $W^r$ and alternative processes for calculating $\hat{P}$ ~\citep{imbens2004nonparametric,schuler2017targeted}. }
\subsection{Fast Subset Scanning for Conditional Bias Scan}
\label{fss_for_cbs_appendix}

\ignore{
OUTLINE FOR EDITING
 -- short intro sentence to section
 -- introduction to computational problem
 -- generalized walk through FSS
 -- specialization for CBS
}

In this section, we explain the fast subset scanning (FSS) algorithm that CBS uses to find the subgroup of the protected class with the most biased predictions or recommendations~\citep{neill2012fast}.\ignore{ We will describe the algorithm at a high level, including functionality that is not specific to CBS, with notation and context specific to its implementation for CBS. } We will introduce FSS using a simplified example, for illustrative purposes, to highlight the computational difficulties inherent in subset scanning, the additive property of the score functions for CBS that enable computationally feasible subset scanning,  and the implementation of FSS for CBS.

Let us assume a dataset of individuals in the protected class ($A=1$), denoted as $Q=\{(X^1,I,\hat{I})\}$, that contains values of the event variable $I_i$, estimates $\hat{I}_i$ of the expected value of the event variable under the null hypothesis of no bias, and  
 a single categorical covariate attribute $X_{i}^1$ for each individual $i$. 
For concreteness, we perform a sufficiency scan for predictions, therefore, the event variable $I_i$ is the observed binary outcome $Y_i$ for individual $i$, and the corresponding $\hat{I}_i$ is the estimated $\Pr(Y_i = 1 \:|\: P_i, X_i)$ under the null hypothesis $H_0$ that $Y \perp A \:|\: (P,X)$.  $S$ refers to a subgroup of $Q$, which in our simple example is a non-empty subset of values for attribute $X^1$.  Since our event variable is binary, we use the Bernoulli likelihood function to represent the hypotheses in the score function, $F(S)$, used to determine the level of anomalousness of a subgroup $S$ of $Q$.


In the worst-case scenario, $X^1$ could be a categorical variable with distinct values for each of the $n$ rows of data in $Q$. If we were to score all of the possible $S \subseteq Q$ using $F(S)$, this method would have a runtime of $O(2^n$), which would be computationally infeasible.  To overcome this computational barrier, FSS relies on its score functions, $F(S)$, being a part of an efficiently optimizable class of functions in order to find the most anomalous subset $S^\ast = \arg\max_{S\subseteq Q} F(S)$ without the need to evaluate all of the subsets of $Q$. The property that determines if a function is a part of this class that enables fast subset scanning is called Additive Linear-Time Subset Scanning (ALTSS)~\citep{speakman2016penalized} and is formally defined in Appendix~\ref{altss_defintion}.  Informally, if $F(S)$ can be represented as an additive set function over all instances $i \in S$ when conditioning on the free parameter ($q$ for the Bernoulli distribution or $\mu$ for the Gaussian distribution in Table~\ref{methods-score_function_table}), it satisfies this property~\citep{speakman2016penalized}.


To explore how FSS exploits the ALTSS property for computationally efficient subset scanning, assume that the categorical covariate $X^1$ for each individual $i$ can only be equal to one of four values, $X_i^1 \in \{a,b,c,d\}$. FSS constructs a subset for each attribute value of $X^1$ such that $S_a=\{i \in Q: X_i^1=a\}$, $S_b=\{i \in Q: X_i^1=b\}$, $S_c=\{i \in Q: X_i^1=c\}$, $S_d=\{i \in Q: X_i^1=d\}$.  Since we are using the likelihood function for the Bernoulli distribution for $F(S)$, $F(S)$ is a concave function of the free parameter $q$, and for illustrative purposes, we will assume that $\max_q F(S)$ is positive for all subsets $S_a,S_b, S_c$ and $S_d$. Therefore, for each subset $S_a,S_b, S_c$ and $S_d$, $F(S)$ is a function over the domain of $q$, where as $q$ increases from $-\infty$, $F(S)$ eventually equals $0$ and then the global maximum for $F(S)$ for that given subset, and then starts decreasing until it again reaches a point where $F(S)=0$, and then remains negative as $q$ approaches $\infty$. FSS identifies three $q$ values for each subset, $S \in \{S_a,S_b,S_c,S_d\}$:
\begin{enumerate}
    \item The first value of $q$ where $F(S)=0$ as $q$ increases from $-\infty$ to $\infty$, which we will refer to as $q_{min}$.
    \item The second value of $q$ where $F(S)=0$ as $q$ increases from $-\infty$ to $\infty$, which we will refer to as $q_{max}$.
    \item The value of $q$ for $\argmax_q F(S)$, which we will refer to as $q_{\text{MLE}}$.
\end{enumerate}

Each distinct $q_{min}$ and $q_{max}$ value for subsets ($S_a$, $S_b$, $S_c$, $S_d$) is a value of $q$ where the score function $F(S)$ becomes negative or positive for at least one of these four subsets. By sorting all of the distinct $q_{min}$ and $q_{max}$ values across all the subsets ($S_a$, $S_b$, $S_c$, $S_d$) in ascending order, we construct a list of $q$ values, $\{q_{(1)}, ..., q_{(m)}\}$, where each pair of adjacent values, $q_{(k)}$ and $q_{(k+1)}$, represents an interval of the $q$ domain, $(q_{(k)},q_{(k+1)})$, for which each subset $S \in \{ S_a,S_b,S_c,S_d\}$ has either $F(S) > 0$ for the entire interval or $F(S) < 0$ for the entire interval. 
For each interval, we perform the following:
\begin{enumerate}
    \item Find the midpoint of the interval (average of $q_{(k)}$ and $q_{(k+1)}$), which we refer to as $q_{k}^{\text{mid}}$. 
    \item Create a new subset $S_{k}^{\text{aggregate}}$ by aggregating all subsets $S \in\{S_a, S_b, S_c, S_d\}$ where the subset's $q_{min} < q_{k}^{\text{mid}}$ and the subset's $q_{max} > q_{k}^{\text{mid}}$, i.e., $F(S) > 0$ when $q = q_{k}^{\text{mid}}$ and therefore for the entire interval ($q_{(k)}$, $q_{(k+1)}$).  
    
    Since the score function is additive, conditioned on $q$, we know that a subset $S$ will make a positive contribution to the score $F(S_{k}^{\text{aggregate}})$ if and only if $F(S) > 0$ for that value of $q$. Thus, we know that the highest scoring subset $S_{k}^{\text{aggregate}}$ for that interval $[q_{(k)},q_{(k+1)}]$ contains all and only those subsets $S$ with $F(S)>0$ at $q=q_{k}^{\text{mid}}$. 
    \item Find the maximum likelihood estimate of $q$, $q_{\text{MLE}}^{\text{aggregate}} = \argmax_q F(S_{k}^{\text{aggregate}})$, and the corresponding score $F(S_{k}^{\text{aggregate}})$. 
\end{enumerate}
\ignore{the midpoint of the interval (average $q$ between the beginning and end of an interval), which we refer to as $q_{\text{mid}}$. From calculating $q_{min}$ and $q_{max}$ for each singular attribute-value subset, we know which singular attribute-value subset has a positive score function $F(S)$ at $q_{\text{mid}}$, and therefore we can create a new subset for each interval that includes all the subsets that have positive score functions at $q_{\text{mid}}$. We will then recalculate the score functions for these new subsets that are associated with the different intervals and find the $q_{\text{MLE}}^{\text{aggregate}}$ that maximizes the score functions for these new subsets constructed from the intervals over the $q$ domain.}The aggregate subset, $S_{k}^{\text{aggregate}}$, with the highest score for $F(S)$ using its associated $q_{\text{MLE}}^{\text{aggregate}}$ is the most anomalous subset when considering subsets formed by combinations of different attribute-values of $X^1$.

For our simplified example, there are at most 8 distinct $q_{min}$ or $q_{max}$ values from the four subsets ($S_a$, $S_b$, $S_c$, $S_d$), and thus at most 7 distinct intervals $(q_{(k)},q_{(k+1)})$ that must be considered.  For a given interval, we need to evaluate only a single subset $S_{k}^{\text{aggregate}}$, and thus, only 7 of the 15 non-empty subsets of $\{S_a, S_b, S_c, S_d\}$.  More generally, if $n$ is the arity (number of attribute values) of categorical attribute $X^1$, at most $2n-1$ of the $2^n - 1$ non-empty subsets of attribute values must be evaluated to identify the highest-scoring subgroup.

The scenario where the covariates consist of a single categorical attribute is a simplified example, where only a single iteration of FSS is needed to find the optimal subset, $S^{*}$, of $Q$.  When there are two or more attributes for the covariates, multiple iterations of FSS must be performed to find the optimal subset.  On each iteration the following is performed: 
\begin{enumerate}
    \item We define an initial subset, $S_{temp}$ where:
    \begin{enumerate}
        \item If it is the first iteration, all of the attribute values for each attribute are included in $S_{temp}$.
        \item Otherwise, a random subset of attribute values for each attribute are chosen to be included in $S_{temp}$.
    \end{enumerate}
    \item For each attribute $X^i$, in random order, we construct subsets by partitioning $S_{temp}$ by the distinct attribute values of $X^i$, form intervals across the domain of $q$ for $F(S)$, and then assemble and score the subsets for each interval (as described above).  $S_{temp}$ is updated as higher scoring subsets using $F(S)$ are found.  Therefore, when an attribute is evaluated, $S_{temp}$ contains only rows of $Q$ that fit the found criteria (in the form of attribute values) from previously evaluated attributes, excluding the attribute currently under consideration.  This iterative ascent procedure is repeated until convergence.
\end{enumerate}
Multiple iterations are performed with the final optimal subset being the subset with the highest score using $F(S)$ found across all the iterations, $S^{\ast}$. For the pseudocode of FSS for CBS, reference Appendix~\ref{pseudo-code_CBS}. The final results from FSS are the optimal subset, $S^{*}$, in the form of attribute-values that form the criteria for the subgroup in the protected class with the most anomalous bias detected, the parameter $q$ or $\mu$ that maximizes $F(S^{*})$, and the score $F(S^{*})$ given the parameter $q$ or $\mu$.

\subsubsection{Formal Definition of Additive Linear-Time Subset Scanning Property (ALTSS)} \label{altss_defintion}

Below we provide a formal definition of the Additive Linear-Time Subset Scanning Property. The score functions, $F(S)$, used to evaluate subgroups are a log-likelihood ratio formed from two different hypotheses whose likelihoods are modeled by likelihood functions for either the Bernoulli distribution or Gaussian distribution, both of which satisfy the Additive Linear-time Subset Scanning Property~\cite{speakman2016penalized, zhang2016identifying}.

\begin{definition}[Additive Linear-time Subset Scanning Property] \label{altss_prop_definition}
A function, $F:S\times\theta\to\mathbb{R}_{\geq 0}$, that produces a score for a subset $S \subseteq D$, where $D$ is a set of data and $\theta = \argmax_{\theta} F(S\:|\:\theta)$, satisfies the Additive Linear-time Subset Scanning Property if $F(S\:|\:\theta) = \sum_{s_i \in S}F(s_i\:|\:\theta)$ where $s_i$ is a subset of $S$ and $\forall s_i, s_j \in S$, where $s_i \neq s_j$, we have $s_i \cap s_j = \emptyset$.
\end{definition}
\ignore{
For the Conditional Bias Scans that use the Bernoulli distribution to form the log-likelihood ratio scan statistic, $\theta$ in Definition~\ref{altss_prop_definition} is the free parameter $q$, where $q$ is restricted to $q>1$ for detecting biases in the positive direction or under-estimation bias, and $q$ is restricted to $0<q<1$ for detecting biases in the negative direction or over-estimation bias. For the Conditional Bias Scans that use the Gaussian distribution to form the log-likelihood ratio scan statistic, $\theta$ in Definition~\ref{altss_prop_definition} is the free parameter $\mu$, where $\mu$ is restricted to $\mu>0$ for detecting biases in the positive direction or under-estimation bias, and $\mu$ is restricted to $\mu<0$ for detecting biases in the negative direction or over-estimation bias.}
We refer to the score functions, $F(S)$, contained in the rightmost column of Table~\ref{methods-score_function_table} as $F(S\:|\:\mu)$ for the score functions that use the Gaussian likelihood function to form hypotheses and $F(S\:|\:q)$ for the score functions that use the Bernoulli likelihood function to form hypotheses. $F(S\:|\:q)$ contains a summation, $\sum_{i \in S}(I_i \log q - \log(q\hat{I_i} - \hat{I_i} + 1))$, that is the sum of individual-specific values derived from $I_i$, $\hat{I_i}$, and $q$. Given that each individual is distinct, $F(S\:|\:q)=\sum_{i\in S} F(s_i\:|\:q)$, where $s_i$ is the subset of $S$ that contains only individual $i$, satisfies the ALTSS property. Similarly, $F(S\:|\:\mu)$ contains a summation, $\sum_{i \in S}\Delta_i$, that is the sum of individual-specific values $\Delta_i$ derived from $I_i$, $\hat{I_i}$, and $\mu$. Therefore $F(S\:|\:\mu) = \sum_{s_i \in S}F(s_i\:|\:\mu)$, where $s_i$ is the subset of $S$ that contains only individual $i$, satisfies the ALTSS property.

\subsubsection{Pseudocode of Fast Subset Scan Algorithm for Conditional Bias Scan} \label{pseudo-code_CBS}

\begin{algorithm}
\caption{Fast Subset Scan for Conditional Bias Scan}\label{alg:cap}
\begin{algorithmic}[1]
\Require $n_{iters} > 0, (X_i, \hat{I_i}$, $I_i)\: \forall i \in D \text{ where } A_i = 1, direction \in \{\text{positive},\text{negative}\}$
\State $S^{*} \gets \{\}$ \label{lines:global1}
\State $Score^{*} \gets -\infty$
\State $\theta^{*} \gets -\infty$ \label{lines:global2}
\For{$j \gets 1 \ldots n_{iters}$}
 \If{$j == 1$} \label{lines:randominitalization}
 \State $S_{temp} \gets \text{all attribute-values for each attribute in } X$
 \Else
 \State $S_{temp} \gets \text{random nonempty subset of attribute-values for each attribute in } X$ 
 \EndIf
\State $\theta_{temp} \gets \argmax_{\theta}(F(S_{temp}\:|\:\theta))$
\State $Score_{temp} \gets F(S\:|\:\theta_{temp})$
\State $n_{attributes} \gets \text{number of attributes in X}$
\State $n_{scanned} \gets 0$ \Comment{mark all attributes as unscanned}
\While{$n_{scanned} < n_{attributes}$} \label{line:whileattr}
 \State $X_{temp} \gets \text{randomly selected attribute that is marked as unscanned}$ \label{line:randattribute}
 \For{$X_{temp_i} \in X_{temp}$} \Comment{for all attribute-values in $X_{temp}$} \label{lines:forloop1}
 \State $S_{X_{temp_i}} \gets S_{temp}^{relaxed} \cap \{ i \in D : X_{temp} = X_{temp_i}\}$ \Comment{see text of Appendix~\ref{pseudo-code_CBS} for definition of $S_{temp}^{relaxed}$}
 \State $\theta_{min_i}, \theta_{max_i} \gets 
 \arg_{\theta}(F(S_{X_{temp_i}}\:|\: \theta)=0)$ \Comment{exception noted in Appendix~\ref{pseudo-code_CBS}}
 \State $\theta_{MLE_i} = \argmax_{\theta}(F(S_{X_{temp_i}}\:|\: \theta))$
 \State $Score_{i} \gets F(S_{temp_i}\:|\:\theta_{MLE_i})$
 \State Adjust $\theta_{min_i}$ and $\theta_{max_i}$ depending on the $direction$ of scan \Comment{ explained in text of Appendix~\ref{pseudo-code_CBS}} \label{lines:forloop1end}
 
 \EndFor
 \State $\theta_{intervals} \gets \{\theta_{min_i}, \theta_{max_i} \forall X_{temp_i} \in X_{temp} \}$ in ascending order \label{line:sortthresholds}\Comment{all values of $\theta$ where $F(S) =0$ $\forall X_{temp_i} \in X_{temp}$, indexed by $\theta_{(k)}$ below}
 \State $Score_{interval} \gets -\infty$
 \State $S_{interval} \gets \{\}$
 \State $\theta_{interval} \gets -\infty$ \Comment{not to be confused with $\theta_{intervals}$}
 \For{$k \gets 1 \ldots length(\theta_{intervals}) -1$}
 \State $S_{k}^{\text{aggregate}} \gets \{\}$
 \State $\theta_{k}^{\text{mid}} \gets \frac{\theta_{(k)} + \theta_{(k+1)}}{2}$
 \For{$X_{temp_i} \in X_{temp}$} \label{lines:interval_scan}
 \If{$ Score_{i} > 0$ and $\theta_{min_i} < \theta_{k}^{\text{mid}}$ and $\theta_{max_i} > \theta_{k}^{\text{mid}}$}
 \State $S_{k}^{\text{aggregate}} \gets S_{k}^{\text{aggregate}} \cup S_{X_{temp_i}}$
 \EndIf \label{lines:interval_scanend}
 \EndFor
 \State $\theta_{k}^{\text{aggregate}} \gets \argmax_{\theta}(F(S_{k}^{\text{aggregate}} \:|\: \theta)) $
 \State $Score_{k}^{\text{aggregate}} \gets F(S_{k}^{\text{aggregate}} \:|\: \theta_{k}^{\text{aggregate}})$
 \If{$Score_{k}^{\text{aggregate}} > Score_{interval}$}
 \State $Score_{interval} \gets Score_{k}^{\text{aggregate}}$
 \State $S_{interval} \gets S_{k}^{\text{aggregate}}$
 \State $\theta_{interval} \gets \theta_{k}^{\text{aggregate}}$
 \EndIf
 \EndFor
 \If{$Score_{temp} < Score_{interval}$}
 \State $Score_{temp} \gets Score_{interval}$
 \State $S_{temp} \gets S_{interval}$
 \State $\theta_{temp} \gets \theta_{interval}$
 \State $n_{scanned} \gets 0$ \Comment{mark all attributes as unscanned}
 \EndIf
 \algstore{myalg}
 \end{algorithmic}
 \end{algorithm}
 \begin{algorithm}
 \begin{algorithmic}[1]
 \algrestore{myalg}
 \State $n_{scanned} \gets n_{scanned} + 1$ \Comment{mark attribute $X_{temp}$ as scanned}
\EndWhile
\If{$Score^{\ast} < Score_{temp}$}
\State $Score^{\ast} \gets Score_{temp}$
\State $S^{\ast} \gets S_{temp}$
\State $\theta^{\ast} \gets \theta_{temp}$
\EndIf
\EndFor
\State \textbf{return} $S^{\ast}, Score^{\ast}, \theta^{\ast}$
\end{algorithmic}
\end{algorithm}

Algorithm~\ref{alg:cap} is the pseudocode for the Fast Subset Scan (FSS) algorithm used in the CBS framework~\citep{neill2012fast}. The algorithm finds the subgroup, $S^{\ast}$, with the most anomalous signal (i.e., the highest score $F(S^\ast)$) in a dataset.  For CBS, this signal is in the form of a bias (according to one of the fairness definitions in Table~\ref{methods-overview_of_scans_scan_table}) against members of the protected class ($A=1$) for subgroup $S^\ast$. The dataset passed to the FSS algorithm by CBS contains only individuals $i$ in the protected class, and FSS compares their values of the event variable $I_i$ to the estimated expectations $\hat{I}_i$ under the null hypothesis of no bias.

At the initialization of FSS, placeholder variables are created that will hold the most anomalous subset ($S^{*}$), and the subset's corresponding information ($\theta^{\ast}$, $Score^{\ast}$), across all iterations (Lines~\ref{lines:global1}-\ref{lines:global2}). At the beginning of an iteration, a random subset is picked (set of attribute-values) as the starting subset, $S_{temp}$, with the exception of the first iteration where the starting subset includes all attribute values, as shown in the if-else statement starting on Line~\ref{lines:randominitalization}.
For each iteration of this algorithm, we repeatedly choose a random attribute to scan (i.e., we scan over subsets of its attribute values) as shown in Lines~\ref{line:whileattr}-\ref{line:randattribute}, until convergence (i.e., when all attributes have been scanned without increasing the score $F(S_{temp})$). 

For each attribute $X_{temp}$ to be scanned, for each of its attribute values $X_{{temp}_i}$, we score the subset $S_{X_{{temp}_i}}$ containing only the records with the given value of that attribute ($X_{temp} = X_{{temp}_i}$), and matching subset $S_{temp}$ on all other attributes in $X$.  We write this as $S_{X_{{temp}_i}} \gets S_{temp}^{relaxed} \cap \{ i \in D : X_{temp} = X_{temp_i}\}$, where 
$S_{temp}^{relaxed}$ is the relaxation of subset $S_{temp}$ to include all values for attribute $X_{temp}$. Along with scoring this attribute-value subset $S_{X_{{temp}_i}}$, we find the two values of $\theta$ where $F(S_{X_{{temp}_i}})=0$, $\theta_{min_i}$ and $\theta_{max_i}$, and the $\theta$ that maximizes $F(S_{X_{{temp}_i}})$, $\theta_{MLE_i}$, with the exception of attribute-value subsets $S_{X_{{temp}_i}}$ that are not positive for any value of $\theta$. This is shown in the for-loop in Lines~\ref{lines:forloop1}-\ref{lines:forloop1end}. 

Line~\ref{lines:forloop1end} states that $\theta_{min_i}$ and $\theta_{max_i}$ must be adjusted according to the direction of the scan to enforce that the found parameters $\theta_{min_i}$ and $\theta_{max_i}$ adhere to the restrictions set by the direction of the scan. The constraints necessary for the scans to detect biases in the positive and negative directions are fully specified in Table~\ref{methods-score_function_table}.  For positive scans that have score functions that utilize the Gaussian likelihood function to form hypotheses, $\theta_{min_i} = \max(0,\theta_{min_i}$) and for negative scans that utilize the Gaussian likelihood function, $\theta_{max_i} = \min(0,\theta_{max_i}$). For positive scans that have score functions that utilize the Bernoulli likelihood function to form hypotheses, $\theta_{min_i} = \max(1,\theta_{min_i}$) and for negative scans that utilize the Bernoulli likelihood function, $\theta_{max_i} = \min(1,\theta_{max_i}$). Attribute-value subsets $S_{X_{{temp}_i}}$ should not be considered when choosing subsets for $S^{\text{aggregate}}$ for positive scans where $\theta_{max_i} < 0$ or $\theta_{max_i} < 1$ for scans using the Gaussian likelihood function or Bernoulli likelihood function in $F(S)$, respectively. Conversely, attribute-value subsets $S_{X_{{temp}_i}}$ should not be considered when choosing subsets for $S^{\text{aggregate}}$ for negative scans where $\theta_{min_i} > 0$ or $\theta_{min_i} > 1$ for scans using the Gaussian likelihood function or Bernoulli likelihood function in $F(S)$, respectively. 

We sort the $\theta_{min_i}$ and $\theta_{max_i}$ values found across all the attribute values of the attribute we are scanning in ascending order in Line~\ref{line:sortthresholds}. These form a list of intervals over the domain of $\theta$. For each interval, we calculate a midpoint of that interval, and aggregate all the attribute-value subsets that have a positive score, $F(S)$, when $\theta$ equals the midpoint of that interval in Lines~\ref{lines:interval_scan}-\ref{lines:interval_scanend}. If the aggregated subset of attribute values with the maximum score across all the intervals is greater than the score of $S_{temp}$, we update $S_{temp}$ and all of its accompanying information ($\theta_{temp}$, $Score_{temp}$) to equal the maximum-scoring subset of aggregated attribute-values across all the intervals and its accompanying information. $S_{temp}$ is continuously updated as higher scoring subsets are found as we scan over all the attributes and their attribute values. 

At the end of an iteration, if the found subset, $S_{temp}$, has a higher score than the global maximum scoring subset $S^{\ast}$, then $S^{\ast}$ and its accompanying information ($\theta^{\ast}$, $Score^{\ast}$) are replaced with $S_{temp}$ and $S_{temp}$'s accompanying information. Once all the iterations have completed, the subset with the maximum score found across all iterations is returned, $S^{\ast}$, with its score $F(S^{\ast}\:|\:\theta^{\ast})$ and accompanying $\theta^{\ast}$ parameter.

McFowland et al. show that a similar multidimensional scan algorithm, used for heterogeneous treatment effect estimation, will converge with high probability to a near-optimal subset when run with multiple iterations~\citep{mcfowland2018efficient}. 

\subsection{Permutation Testing to Determine Statistical Significance of Detected Subgroups} \label{permutation_testing}

As discussed in Section~\ref{methods_cbs_score_functions_altss_section}, the statistical significance ($p$-value) of the discovered subgroup $S^\ast$ can be obtained by \emph{permutation testing}, which correctly adjusts for the multiple testing resulting from searching over subgroups. To do so, we generate a large number of simulated datasets under the null hypothesis $H_0$, perform the same CBS scan for each null dataset (maximizing the log-likelihood ratio score over subgroups, exactly as performed for the original dataset), and compare the maximum score $F(S^\ast)$ for the true dataset to the distribution of maximum scores $F(S^\ast)$ for the simulated datasets. The detected subgroup is significant at level $\alpha$ if its score exceeds the $1-\alpha$ quantile of the $F(S^\ast)$ values for the simulated datasets.  To generate each simulated dataset, we copy the original dataset and randomly permute the values of $A_i$ (whether or not each individual is 
a member of the protected class), thus testing the null hypothesis that $A$ is conditionally independent of the event variable $I$.

This permutation testing approach is computationally expensive, multiplying the runtime by the total number of datasets (original and simulated) on which the CBS scan is performed, but it has the benefit of bounding the overall false positive rate (family-wise type I error rate) of the scan while maintaining high detection power.  (In comparison, a simpler approach like Bonferroni correction would also bound the overall false positive rate, and would require much less runtime, but would suffer from dramatically reduced detection power.)  For a given dataset, the score threshold for significance at a fixed level $\alpha=.05$ will differ for different choices of the sensitive attribute and protected class. Thus, if CBS is used to audit a classifier for possible biases against multiple protected classes, a separate permutation test must be performed for each protected class value.

\clearpage
\subsection{Conditional Bias Scan Framework Parameters} \label{framework_cbs_parameters}

\begin{table}[h!]
\scalebox{.825}{
\begin{tabular}{|p{3.5cm}|p{5.5cm}|p{4cm}|p{2cm}|}
\hline
Parameter & Purpose & Parameter Attribute Values & Sections for Reference \\
\hline
Membership in Protected Class Indicator Variable ($A$)& Binary attribute which defines whether each individual is a member of the protected class. We wish to identify any biases that are present in the classifier's predictions or recommendations that impact the protected class. && \ref{methods} \\
\hline
Scan Type & The subcategory of the scan type & Separation scan for recommendations; Separation scan for predictions; 
Sufficiency scan for recommendations; 
Sufficiency scan for predictions & \ref{methods-overview_of_scan_types} \\
\hline
Event Variable ($I$) & The event of interest for the scan. The abstracted event variable must be defined as either the outcome, prediction, or recommendation variable. & $Y$; $P$; $P_{bin}$ & \ref{methods}, \ref{methods-overview_of_scan_types} \\
\hline
Conditional Variable ($C$) & The conditional variable for the scan. The abstracted conditional variable must be defined as either the outcome, prediction, or recommendation variable. & $Y$; $P$; $P_{bin}$ & \ref{methods}, \ref{methods-overview_of_scan_types} \\
\hline
Field value ($z$) of Conditional Variable ($C=z$) & For value-conditional scans, this is the value on which we are conditioning the conditional variable ($C$). Defining a field value results in scans that detect different forms of fairness violations. & None; 0; 1 & \ref{methods}, \ref{methods-generating-hat-p-section}, \ref{methods_cbs_score_functions_altss_section}, \ref{fss_for_cbs_appendix} \\

\hline

List of Attributes for forming subgroups ($X$) & List of attributes to scan over to form subgroups & &\ref{methods}, \ref{methods-overview_of_scan_types}, \ref{fss_for_cbs_appendix}\\
\hline
Direction of Bias & Specifying whether we are detecting under-estimation bias (positive direction) or over-estimation bias (negative direction) & Positive; Negative & \ref{methods-overview_of_scan_types}, \ref{methods_cbs_score_functions_altss_section}, \ref{fss_for_cbs_appendix} \\
\hline
List of Attributes for estimating $\hat{I}$ ($X$) & List of attributes used for conditioning when producing $\hat{I}$. In this paper we use the same attributes to form subgroups and produce $\hat{I}$. This does not necessarily have to be the case for all applications of CBS. & & \ref{methods-generating-hat-p-section}, \ref{generating_i_hat_appendix} \\
\hline
Subgroup Complexity Penalty & The non-negative integer-valued scalar penalty that is subtracted from the score function for each subgroup, depending on the subgroup's total number of included values for each covariate $X^1 \ldots X^m$, not including covariates for which all values are included. & $0+$ & \ref{methods_cbs_score_functions_altss_section} \\
\hline
Scan Iterations & Specifies the number of iterations to run the fast subset scanning algorithm & $1+$ &  \ref{methods_cbs_score_functions_altss_section}, \ref{fss_for_cbs_appendix}\\
\hline 
\end{tabular}}
\caption{Table with all parameters needed to run Conditional Bias Scan. The table lists the parameter, purpose of the parameter, possible values of the parameter, when applicable, and the sections in our paper where this parameter is described in further detail.}

\label{methods-parameter-table}

\end{table}

\newpage
\section{Evaluation Appendices}

\subsection{Adaptions of the Competing Methods used as Benchmarks} \label{benchmark_methods_adaptions}

Both GerryFair and Multiaccuracy Boost provide implementations of their methods on GitHub~\cite{gerryfair_github, multiaccuracy_boost_github}. Our goal was to use their provided code with minimal changes as benchmarks in Section~\ref{experiments_validation}. However, GerryFair and Multiaccuracy Boost do not provide the functionality to indicate whether to audit for bias in the positive direction (under-estimation bias) or bias in the negative direction (over-estimation bias). This lack of functionality makes the results from CBS substantially different than those returned by GerryFair and Multiaccuracy Boost. 

For GerryFair's auditor, given the type of error rate to audit (false negative rate or false positive rate), they train four linear regressions using the features ($X$) as dependent variables with the following four sets of labels: 
\begin{enumerate}
    \item Two linear regressions with the zero set as labels.
    \item One linear regression with the labels set to a measurement that assigns positive costs for predictions that deviate in the \emph{positive} direction (when the predictions are greater than the observed global error rate), and negative costs otherwise.
    \item One linear regression with the labels set to a measurement that assigns positive costs for predictions that deviate in the \emph{negative} direction (when the predictions are less than the observed global error rate), and negative costs otherwise.
\end{enumerate}
\ignore{and the labels for two linear regressions being the zero set and the labels for two of the linear regressions being two distinct calculations of the negated normalized deviation of the predictions from the observed baseline for the metric of interest (either false negative rate or false positive rate). }
They use the predictions from the linear regressions to flag a subset of data where the predictions from the linear regression trained with the zero set labels are greater than the values predicted by the linear regression trained with the costs representing deviations of the predictions from the observed baseline error rate metric of interest as labels. Two linear regressions are used to estimate deviations of the predictions from the observed error rate baseline\ignore{ for the metric of interest because there are two ways they calculate this deviation}, and therefore they form two subgroups: (1) a subgroup with rows that are estimated to have predictions that are greater than the baseline for the metric of interest; and (2) a subgroup with rows that are estimated to have predictions that are less than the baseline for the metric of interest\ignore{ based on the two ways of calculating the negated normalized deviation of the predictions from the observed baseline metric of interest}. The original GerryFair implementation uses a heuristic to decide which subgroup has more significant biases and returns that subgroup accordingly. The subgroup with the rows that are estimated to have predictions that are greater than the metric of interest more closely aligns with the concept of auditing for bias in the positive direction or auditing for under-estimation bias. Since CBS provides the functionality of auditing for biases of a specific direction, we add an option to GerryFair that allows the user to determine which direction of bias they are interested in, making GerryFair's results more comparable to CBS. \ignore{only return one of the subgroups that is estimated to have predictions that are greater than the baseline false positive error rate. }

For each simulation, we ran GerryFair two times, once to detect bias in the form of systematic increases in the false positive rate, and once to detect bias in the form of systematic increases in the false negative rate.  In each case, we allow GerryFair to use all covariates ($X$) to make the predictions used to form subgroups, including the protected class category.
\ignore{We run GerryFair to detect bias in the form of false positives and false negatives as well as marking only the protected class attribute value as a sensitive attribute-value and marking all the attribute-values as sensitive attributes.} This resulted in two result sets for GerryFair for each simulation. We present the result set in Section~\ref{experiments_validation} that had the highest overall accuracy for most of the simulations, which is the GerryFair setup for detecting increased false positive rate. GerryFair returns a subgroup that could contain individuals in both the protected class and the non-protected class. To have the accuracy measurements for GerryFair and CBS be comparable, we filter the subgroup returned by GerryFair to only include individuals in the protected class before calculating the subgroup's accuracy.

\ignore{Unlike GerryFair that provided a standalone auditor method, }
Multiaccuracy Boost is an iterative algorithm, where on each iteration it audits for a subgroup with inaccuracies and then corrects that subgroup's predicted log-odds. More specifically, for each iteration:
\begin{enumerate}
    \item \label{custom_heuristic} A custom heuristic is calculated for all rows of data, similar to an absolute residual, where larger values represent a larger deviation between the observed labels and predictions.
    \item The residuals of all the rows' predictions and observed outcomes are calculated.
    \item The full data is split into a training and holdout set.
    \item \label{set_partitions_multiacc} Three partitions of data are created for the training data, hold out data, and the full dataset:
    \begin{enumerate}
        \item A partition containing all the rows.
        \item A partition containing all the rows with predictions greater than 0.50.
        \item A partition containing all the rows with predictions less than or equal to 0.50.
    \end{enumerate}
    \item For each of the partitions of data constructed in Step \ref{set_partitions_multiacc}:
    \begin{enumerate}
        \item A ridge regression classifier (using $\alpha=1.0$) is trained using the respective partition in the training data, with the covariates $X$ as features and the custom heuristic calculated in Step~\ref{custom_heuristic} as labels.
        \item The ridge regression classifier is used to make predictions for the respective partition in the holdout data.
        \item \label{score_multiacc} If the average of the predictions multiplied by the residuals for the partition set in the hold out data is greater than $10^{-4}$, then the predicted log-odds for the respective partition in the full dataset is shifted by the predictions multiplied by 0.1.
        \item If the predicted log-odds are updated, the iteration terminates and no other partitions of data are evaluated for that iteration.
    \end{enumerate}
\end{enumerate}
\ignore{the data is partitioned into a training and holdout set. The residuals are calculated for all the rows in the training set. The training data is further partitioned the training data into three sets (as well as the holdout data, and full data set at large), one containing all the rows with predictions greater than 0.5, one containing all the rows with predictions less than or equal to 0.5, and another that contains all the rows. For each of these three partitions they fit a linear regression where the loss function is the linear least squares function with an L2 penalty (i.e., ridge regression) using the covariates ($X$) as features and a customized heuristic that measures the deviation between a prediction and its outcome as labels. It uses this linear regression that is fit with the training data to estimate the distribution of this custom heuristic on the holdout set's equivalent partition of data. It then produces a score by averaging the product of the custom heuristic and the residuals of the given partition in the holdout set. If that score falls above a predefined threshold, the predicted log-odds for the equivalent partition of the full data are adjusted to be more accuracy. Once the predicted log-odds are updated that iteration terminates, and no other partitions for that iteration are evaluated.} 

The steps above are slightly modified for the scenario of a classifier that produces a singular probability of a positive outcome whereas the original Multiaccuracy Boost was designed for was a bivariate outcome vector from a Inception-ResNet-v1 model. To make Multiaccuracy Boost audit for bias in one direction, when calculating whether a partition of the data's predicted log-odds should be updated using the holdout data to remove an inaccuracy, we override the residuals that are negative with 0. In effect, we only consider rows with negative outcomes when deciding which partition of predictions have inaccuracies that need to be corrected on a given iteration. This was the least invasive modification we could make to Multiaccuracy Boost to have it solely consider bias in the positive direction when deciding which subgroup's predicted log-odds to update.

Since the auditor and correction method are functioning in tandem, we run all iterations of the algorithm and log each subgroup (i.e., partition) that was detected as needing a correction to its predicted log-odds and its associated score calculated in Step~\ref{score_multiacc}. After the algorithm terminates, we find the partition with the highest score and return its associated partition in the full data set. The decision to return the partition with the highest score across all the iterations of Multiaccuracy Boost in the simulations is motivated by the fact that Multiaccuracy Boost's auditor has no theoretical guarantees of detecting the most inaccurate partition on a specific iteration of the algorithm, but given our intuition of the algorithm this partition is most likely to be found on the first iteration of the algorithm.\ignore{ Because of how the partitions are formed, and the property that the algorithm exhibits of overall improvement of accuracy on each iteration this is most likely the subgroup found in the first partition.} Similarly to GerryFair, Multiaccuracy Boost detects a subgroup that contains members of the protected class and non-protected class. We filter all the individuals in the returned subgroup to only contain individuals who are part of the protected class before calculating the accuracy of the returned partition.

One distinction between these methods and CBS is that their auditors were intended to be used in conjunction with another process to improve a classifier or predictions. Therefore, their auditors were designed to have the level of detection accuracy necessary to discern which subgroups or partitions of data need to be corrected, either by modifying the classifier or by post-processing their predicted log-odds. Given that both methods suggest that they can be used for auditing purposes, they are appropriate choices as benchmarks for CBS, but it is important to note that CBS was specifically designed to have a high accuracy for bias detection, whereas that was not necessarily an explicit intention of GerryFair or Multiaccuracy Boost.

\subsection{Explanation of the Additive Term ($\epsilon^{true}$) for the True Log-Odds used in the Generative Model for the Semi-synthetic Data} \label{additive_term_validation}

For the evaluation simulations described in Section~\ref{experiments_validation}, when producing the true log-odds that are used to determine the outcomes and predicted values, we add a term to each row's true log-odds of a value drawn from a Gaussian distribution $\epsilon_{i}^{true} \sim \mathcal{N}(0,\sigma_{true})$ where $\sigma_{true} = 0.6$. We add this term to the true log-odds to ensure that when the true log-odds for the rows of $S_{bias}$ in the protected class are injected with $\mu_{suf}$, this results in a violation of the fairness definition for sufficiency. 

For the remainder of this section we will focus on sufficiency scan for predictions, but our explanation below is applicable for sufficiency scan for recommendations as well. Sufficiency implies that the outcomes $Y$ are conditionally independent of membership in the protected class $A$ given the predictions $P$ and covariates $X$, that is, $Y \perp A \:|\: (P,X)$.
 Assume that we have predictions that are independent of the outcome conditional on the covariates, $Y \perp P \:|\: X$. Since the outcome is independent of the predictions conditional on the covariates, the definition of sufficiency simplifies to $Y \perp A \:|\: X$. This simplification of sufficiency reduces sufficiency scans to finding the subgroup in the protected class with the biggest base rate difference from its corresponding subgroup in the non-protected class regardless of that subgroup's predictions.  Therefore, it is not evaluating sufficiency violations because these base rate differences are independent of the predictions. Consequentially, when there is \emph{no} base rate difference between the protected and non-protected class conditional on the covariates, ($Y \perp A \:|\: X$), in order for sufficiency to be violated, $Y \not\perp A \:|\: (P,X)$, we must also have $Y \not\perp P \:|\: X$.  This is formally stated in Theorem~\ref{sufficiency_statement}.

 \begin{theorem}\label{sufficiency_statement}
 To have violations of the sufficiency definition, $Y \not\perp A \:|\: (P,X)$, when there are no base rate differences between the protected class and non-protected class conditional on the covariates, $Y \perp A \:|\: X$, the predictions and outcomes must be conditionally dependent given the covariates, $Y \not\perp P \:|\: X$.
 \end{theorem}

 \begin{proof}
  Let us assume that (i) there are no base rate differences between protected and non-protected class conditional on the covariates, $Y \perp A \:|\: X$; (ii) outcomes are independent of the predictions conditional 
 on the covariates, $Y \perp P \:|\: X$; and (iii) violations of the sufficiency definition exist, $Y \not\perp A \:|\:(P,X)$.  We will show that these three statements lead to a contradiction.  First, 
$( Y \perp P \:|\: X)$ and $(Y \perp A \:|\: X )$ together imply that $Y \perp (P , A) \:|\: X$. Furthermore, using the weak union axiom for conditional independence, $Y \perp (P , A) \:|\: X$ implies that $Y \perp A \:|\: (P, X)$, which contradicts (iii).  Since these three statements cannot all be true, we know that no base rate differences (i) and violations of sufficiency (iii) together imply that the outcomes cannot be independent of the predictions conditional on the covariates, $Y \not\perp P \:|\: X$.
\end{proof}

 To ensure that $Y \not\perp P \:|\: X$, the predictions $P$ must carry information about the outcomes $Y$ that is not carried in $X$. By adding the term $\epsilon_{i}^{true}$ to the true log-odds for each row, given that the predicted log-odds (and the corresponding predicted probabilities $P_i$ and binarized recommendations $P_{i,bin}$) and the outcomes $Y$ are both derived from the true log-odds, this ensures that $Y \not\perp P \:|\: X$ in the evaluation simulations because $P$ carries information about $Y$, in the form of the added row-wise terms (drawn from a Gaussian distribution), that are captured in $Y$, but are not captured in $X$.

\newpage
 \subsection{Additional Evaluation Simulations} \label{varying_evaluation}

\begin{figure}
 \centering
 \includegraphics[width=\textwidth]{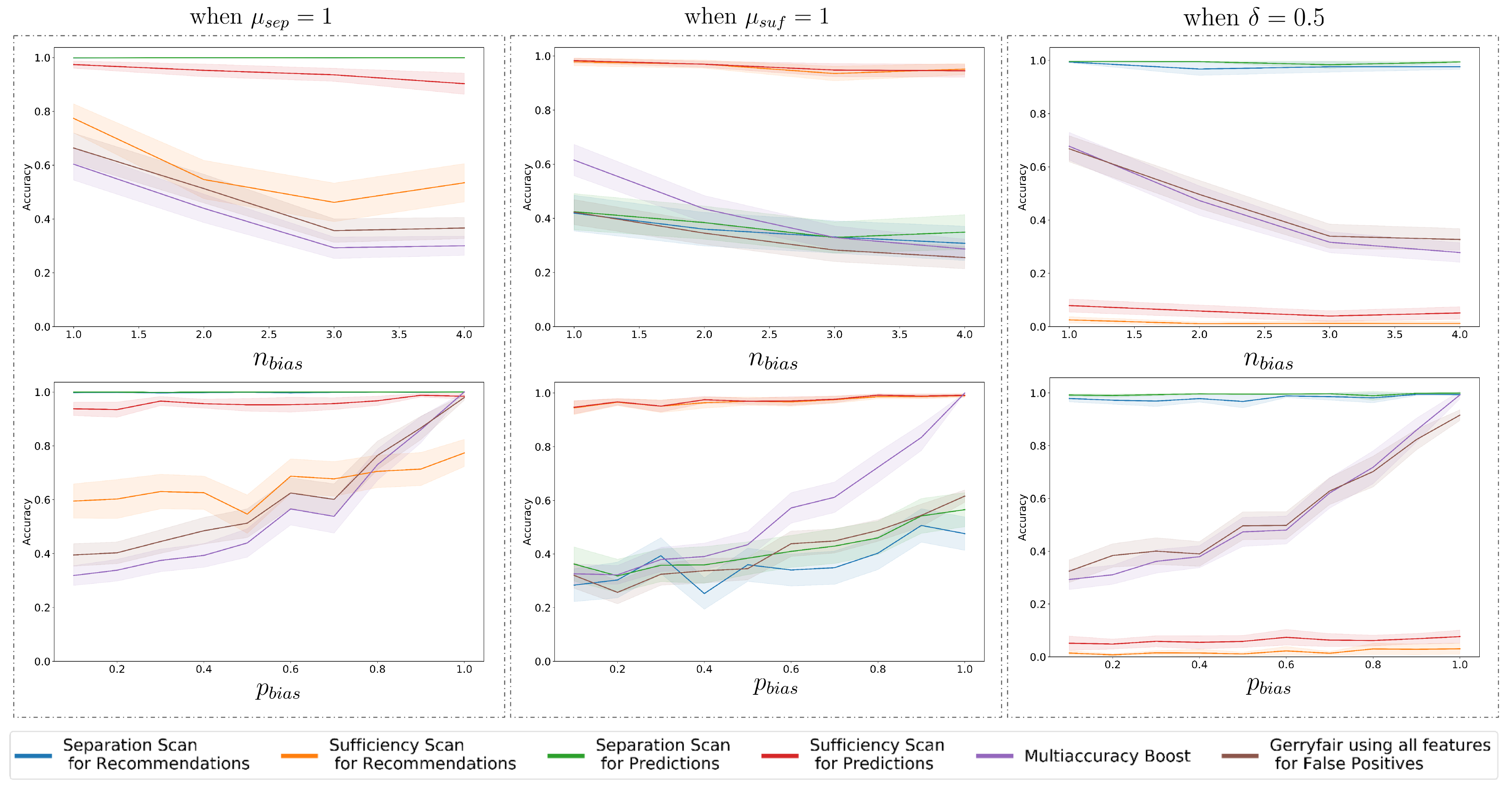}
 \caption{Average accuracy (with 95\% CI) for biases and base rate shifts injected into subgroup $S_{bias}$ of the protected class, for CBS, GerryFair, and Multiaccuracy Boost, as a function of varying parameters $n_{bias}$ (top row) and $p_{bias}$ (bottom row). Left: increasing predicted log-odds by $\mu_{sep}=1$. Center: decreasing true log-odds by $\mu_{suf}=1$. Right: base rate difference $\delta = 0.5$, for $\mu_{sep}=\mu_{suf}=0$.}
 \label{cd_diagrams}
\end{figure} 

\begin{figure}
 \centering
 \includegraphics[width =\textwidth]{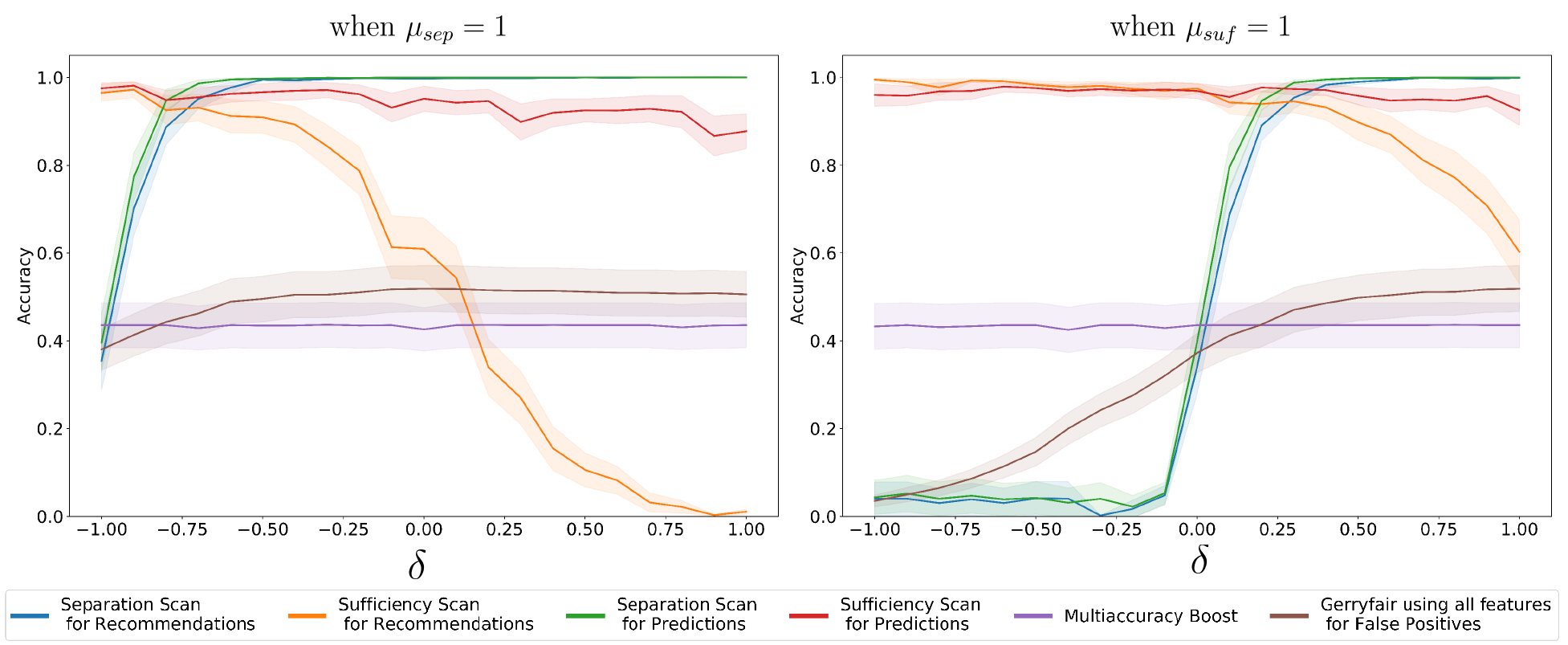}
 \caption{Average accuracy (with 95\% CI) for biases injected into subgroup $S_{bias}$ of the protected class, for CBS, GerryFair, and Multiaccuracy Boost, as a function of varying base rate difference $\delta$ between protected and non-protected class for subgroup $S_{bias}$. Left: increasing predicted log-odds by $\mu_{sep}=1$. Right: decreasing true log-odds by $\mu_{suf}=1$. }
 \label{_3_diagrams}
\end{figure}

To evaluate \ref{Q3} in Section~\ref{experiments_validation}, we modify the characteristics of $S_{bias}$, by varying $n_{bias}$ and $p_{bias}$ for three settings, when $\mu_{sep} =1$, $\mu_{suf} =1$, and $\delta=0.5$.  For each setting, we perform two simulations: (1) by varying the number of attribute categories to choose attribute-values from ($n_{bias}$) between 1 and 4, when $p_{bias} = 0.50$; and (2) by varying the probability ($p_{bias}$) of an attribute-value being included in $S_{bias}$ between 0 and 1, when $n_{bias} =2$.  The results of these simulations are shown in Figure~\ref{cd_diagrams}. We observe that, when varying $n_{bias}$, CBS has similar accuracy results to the simulations shown in Figures~\ref{a_diagrams} and \ref{2a_diagrams}, with separation scans and sufficiency scan for predictions having higher bias detection accuracy when $\mu_{sep}=1$, and sufficiency scans having higher bias detection accuracy when $\mu_{suf}=1$, as compared to competing methods across all settings of $n_{bias}$.
Interestingly, GerryFair and Multiaccuracy Boost have improved bias detection accuracy, approaching that of CBS, when $p_{bias}$ approaches 1 (i.e., more individuals in the protected class are included in $S_{bias}$), but perform poorly for low values of $p_{bias}$.  This suggests that CBS is better at detecting smaller, more subtle subgroups $S_{bias}$ than the competing methods.

Additionally, we investigated the case where we have both an injected bias ($\mu_{sep}=1$ or $\mu_{suf}=1$) and a base rate shift $\delta$ in subgroup $S_{bias}$ for the protected class. We examined the extent to which positive and negative shifts $\delta$ either help or harm the detection accuracy of the various methods. Thus we run two separate sets of experiments with injected bias $\mu_{sep} = 1$ and injected bias $\mu_{suf} = 1$, while varying the base rate shift $\delta$ from -1 to +1 for each experiment. A positive $\delta$ means $S_{bias}$ in the protected class has a higher base rate, while a negative $\delta$ means $S_{bias}$ in the protected class has a lower base rate, as compared to $S_{bias}$ in the non-protected class.

In Figure~\ref{_3_diagrams}, we observe that the detection accuracy of the separation scans
monotonically increases with $\delta$. This relationship is particularly strong for the experiments with injected bias $\mu_{suf}=1$, in which the separation scans show near-perfect accuracy for large positive $\delta$ and near-zero accuracy for large negative $\delta$. These results are not surprising given the separation scans' sensitivity to positive base rate differences for $S_{bias}$ in the protected class even when no injected bias is present (see Figure~\ref{2a_diagrams} above).
We observe that the detection accuracy of the sufficiency scan for recommendations monotonically decreases with $\delta$. This 
relationship is particularly strong for the experiments with injected bias $\mu_{sep}=1$, in which the sufficiency scan for recommendations shows near-perfect accuracy for large negative $\delta$ and near-zero accuracy for large positive $\delta$. Again, these results are not surprising given the sufficiency scan for recommendations' sensitivity to negative base rate differences for $S_{bias}$ in the protected class even when no injected bias is present (see Figure~\ref{2a_diagrams} above). Finally, we observe that the sufficiency scan for predictions maintains high accuracy for both $\mu_{sep}=1$ and $\mu_{suf}=1$ regardless of the base rate difference $\delta$ for $S_{bias}$ in the protected class.

\newpage
\subsection{Robustness Analyses of Evaluation Simulations for Parameters $\sigma_{true}$ and $\sigma_{predict}$ }\label{additional_validation_setups}
\begin{figure}
 \centering
 \includegraphics[width = \textwidth]{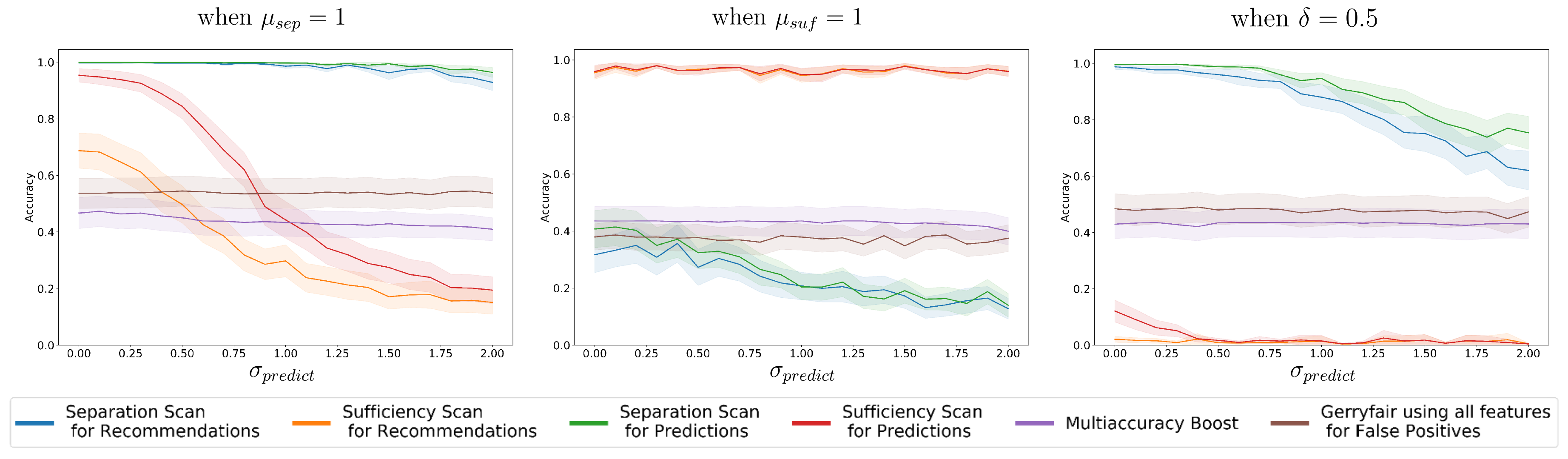}
 \caption{Average accuracy (with 95\% CI) for biases and base rate shifts injected into subgroup $S_{bias}$ of the protected class, for CBS, GerryFair, and Multiaccuracy Boost, as a function of varying parameter $\sigma_{predict}$. Left: increasing predicted log-odds by $\mu_{sep}=1$. Center: decreasing true log-odds by $\mu_{suf}=1$. Right: base rate difference $\delta = 0.5$, for $\mu_{sep}=\mu_{suf}=0$. }
 \label{b_diagrams}
\end{figure}

In this section, we examine the robustness of our results in Section~\ref{experiments_validation} by varying the parameters $\sigma_{predict}$ and $\sigma_{true}$ from their default values of 0.2 and 0.6 respectively. \ignore{We consider the case where we have both an injected bias ($\mu_{sep}=1$ or $\mu_{suf}=1$) and a base rate shift $\delta=.5$ in the affected subgroup $S_{bias}$ of the protected class.} 

First, we examine the impact of varying $\sigma_{predict}$. Recall that each predicted log-odds is drawn from a Gaussian distribution centered at the true log-odds, with standard deviation $\sigma_{predict}$. Thus $\sigma_{predict}$ can be interpreted as the average amount of random error in the classifier's predictions as compared to the true log-odds values.  We run three separate sets of experiments where we alter $S_{bias}$ in the protected class by injecting a bias of $\mu_{sep} = 1$, injecting a bias of $\mu_{suf} = 1$, and creating a base rate difference of $\delta = 0.5$ respectively, while varying $\sigma_{predict}$ between 0 and 2. \ignore{   We re-run the experiments for~\ref{Q1} and~\ref{Q2}, averaging accuracy results over an additional 100 semi-synthetic datasets for each experiment. However, instead of varying $\mu_{sep}$, $\mu_{suf}$, and $\delta$ with a fixed $\sigma_{predict} = 0.2$,}  Accuracies are averaged over 100 semi-synthetic datasets for each experiment. The experiments where $\mu_{sep}=1$ and $\mu_{suf}=1$ analyze the robustness to $\sigma_{predict}$ of the evaluation simulations for \ref{Q1}, whereas the experiments where $\delta=0.5$ analyze the robustness to $\sigma_{predict}$ of the evaluation simulations for \ref{Q2}.

In Figure~\ref{b_diagrams}, we observe that large amounts of noise $\sigma_{predict}$ harm the accuracy of the separation scans for injected biases $\mu_{suf}=1$ which shift the true log-odds in subgroup $S_{bias}$ for the protected class, as well as reducing their detection of base rate shifts $\delta > 0$ in subgroup $S_{bias}$ for the protected class. Most interestingly, when $\sigma_{predict}$ is large, we see a substantial reduction in accuracy for the sufficiency scans for injected biases $\mu_{sep}=1$ which shift the predicted log-odds in subgroup $S_{bias}$ for the protected class. We believe that the combination of noisy predictions and a relatively small dataset size is causing CBS's logistic regression model (used to estimate $\hat{I}$) to underestimate the strength of the relationship between the predictions $P_i$ (or recommendations $P_{i,bin}$) and the outcomes $Y_i$ for the non-protected class. Thus, when we have a large shift in $P_i$ in subgroup $S_{bias}$ for the protected class, the model predicts a much smaller shift in $\mathbb{E}[Y_i \:|\: X_i, P_i]$. As a result, the difference between the expected $\hat{I}_i$ and true $I_i$ values and the corresponding CBS scores for the sufficiency scans are reduced, leading to reduced detection accuracy.

\begin{figure}[t]
 \centering
 \includegraphics[width = \textwidth]{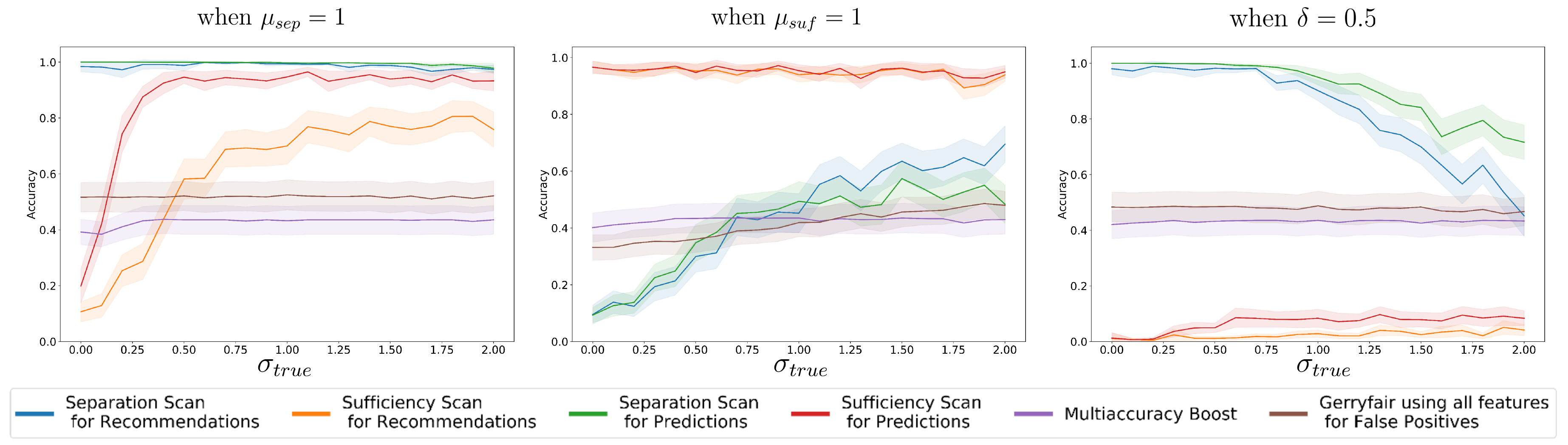}
 \caption{Average accuracy (with 95\% CI) for biases and base rate shifts injected into subgroup $S_{bias}$ of the protected class, for CBS, GerryFair, and Multiaccuracy Boost, as a function of varying parameter $\sigma_{true}$. Left: increasing predicted log-odds by $\mu_{sep}=1$. Center: decreasing true log-odds by $\mu_{suf}=1$. Right: base rate difference $\delta = 0.5$, for $\mu_{sep}=\mu_{suf}=0$. }
 \label{e_diagrams}
\end{figure}

Second, we examine the impact of varying $\sigma_{true}$. Recall that each individual's true log-odds is a deterministic (linear) function of their covariate values $X_i$ plus a term, $\epsilon_i^{true}$, drawn from a Gaussian distribution centered at 0 with a standard deviation of $\sigma_{true}$. Thus the parameter $\sigma_{true}$ represents the variation between individuals' true log-odds based on characteristics other than the covariate values $X_i$ used by CBS. Moreover, since each individual's predicted log-odds is drawn from a Gaussian distribution centered at the true log-odds, these characteristics are assumed to be known and incorporated into the classifier, thus creating the dependency $Y \not\perp P \:|\: X$ when $\sigma_{true} > 0$. In other words, $\sigma_{true}$ represents the average amount of signal in the predictions $P$ (for predicting the outcome $Y$) that is not already present in the covariates $X$.\ignore{ We re-run the experiments for~\ref{Q1} and~\ref{Q2}, averaging accuracy results over an additional 100 semi-synthetic datasets for each experiment. However, instead of varying $\mu_{sep}$, $\mu_{suf}$, and $\delta$ with a fixed $\sigma_{true} = 0.6$, we run three separate sets of experiments with injected bias $\mu_{sep} = 1$, injected bias $\mu_{suf} = 1$, and base rate difference $\delta = 0.5$ respectively, varying $\sigma_{true}$ between 0 and 2 for each experiment. }
 We run three separate sets of experiments where we alter $S_{bias}$ in the protected class by injecting a bias of $\mu_{sep} = 1$, injecting a bias of $\mu_{suf} = 1$, and creating a base rate difference of $\delta = 0.5$ respectively, while varying $\sigma_{true}$ between 0 and 2 for each experiment. Accuracies are averaged over 100 semi-synthetic datasets for each experiment. The experiments where $\mu_{sep}=1$ and $\mu_{suf}=1$ analyze the robustness to $\sigma_{true}$ of the evaluation simulations for \ref{Q1}, whereas the experiments where $\delta=0.5$ analyze the robustness to $\sigma_{true}$ of the evaluation simulations for \ref{Q2}.

In Figure~\ref{e_diagrams}, we observe that small values of $\sigma_{true}$ harm the accuracy of the separation scans for injected bias $\mu_{suf}=1$ while making them more likely to detect base rate shifts $\delta > 0$ in subgroup $S_{bias}$ for the protected class.
Most interestingly, when $\sigma_{true}$ is small, we see a substantial reduction in accuracy for the sufficiency scans for injected bias $\mu_{sep}=1$. This reduced performance for $\sigma_{true} \approx 0$ follows from our argument in Section~\ref{additive_term_validation} above: $\sigma_{true} = 0$ implies $Y \perp P \:|\: X$, and if we also have no base rate difference between the protected and non-protected classes ($Y \perp A \:|\: X$), this implies
$Y \perp A \:|\: P, X$. In other words, even if a bias is injected into the predicted probabilities (and recommendations) in subgroup $S_{bias}$ for the protected class, the sufficiency-based definition of fairness is not violated, and thus the injected bias cannot be accurately detected.

\subsection{Estimates of Compute Power} \label{simulations_compute_power}

 For all of the experiments in Section~\ref{experiments_validation}, Appendix~\ref{varying_evaluation}, and Appendix~\ref{additional_validation_setups}, we used 15 shared, university compute servers running CentOS with 16-64 cores and 16-256 GB of memory.  Each server performed 15-120 runs of CBS concurrently, and ran for approximately 9 days. We estimate that, to run all of the simulations (1,504 CBS runs in total) for a single data set, this would take approximately 32.5 hours. Lastly, to run an individual CBS scan for the COMPAS data (150 iterations), it takes on average approximately 90 seconds.
\clearpage

\section{Case Study of COMPAS Appendices}

\subsection{Additional Information about Preprocessing of COMPAS Data}\label{compas_preprocessing}

We follow many of the processing decisions made in the initial ProPublica analysis, including removing traffic offenses and defining recidivism as a new arrest within two years of the initial arrest for a defendant~\citep{larson_angwin_kirchner_mattu_2016, larson_roswell_2017}. After preprocessing the initial data set, we have 6,172 defendants, their gender, race, age (Under 25 or 25+), charge degree (Misdemeanor or Felony), prior offenses (None, 1 to 5, or Over 5), predicted recidivism risk score (1-10), and whether they were re-arrested within two years of the initial arrest.

\subsection{Full Results of COMPAS Case Study} \label{compas_full_results_appendix}
\begin{longtable}{ |p{1.65cm}|p{1.4cm}|p{2.5cm}|p{2.5cm}|p{0.75cm}|p{1.17cm}|p{1.17cm}|}
\hline
Scan Type & Protected Class Attribute Value & Detected Subgroup & Comparison Subgroup & Score &Observed Rate (Detected)&Observed Rate (Comparison)\\
\hline
&\cellcolor{lightGrayTable}Under age 25&\cellcolor{lightGrayTable}Defendants under age 25 arrested on felony charges (403)&\cellcolor{lightGrayTable}Defendants age 25+ arrested on felony charges (1583)&\cellcolor{lightGrayTable}\textbf{114.1}&\cellcolor{lightGrayTable}0.52&\cellcolor{lightGrayTable}0.39\\
\hhline{~------}
&6+ priors&All defendants with 6+ priors (349)&All defendants with 0-5 priors (3014)&\textbf{83.1}&0.54&0.38\\
\hhline{~------}
&\cellcolor{lightGrayTable}Black&\cellcolor{lightGrayTable}Black male defendants (1168)&\cellcolor{lightGrayTable}Non-Black male defendants (1433) &\cellcolor{lightGrayTable}\textbf{41.9}&\cellcolor{lightGrayTable}0.45&\cellcolor{lightGrayTable}0.35\\
\hhline{~------}
Separation Scan for Predictions&Female&Female defendants with 6+ priors (40)&Male defendants with 6+ priors (309)&25.3&0.59&0.54\\
\hhline{~------}
&\cellcolor{lightGrayTable}1 to 5 priors&\cellcolor{lightGrayTable}Defendants under age 25 with 1 to 5 priors (227)&\cellcolor{lightGrayTable}Defendants under age 25 with 0 or 6+ priors (366)&\cellcolor{lightGrayTable}\textbf{3.23}& \cellcolor{lightGrayTable}0.54&\cellcolor{lightGrayTable}0.49\\
\hhline{~------}
&Felony&White female defendants arrested on felony charges (139)&White female defendants arrested on misdemeanor charges (173)&\textbf{2.40}&0.42&0.34\\
\hhline{~------}
&\cellcolor{lightGrayTable}Male&\cellcolor{lightGrayTable}Asian male defendants (22)&\cellcolor{lightGrayTable}Asian female defendants (1)&\cellcolor{lightGrayTable}0.62&\cellcolor{lightGrayTable}0.30&\cellcolor{lightGrayTable}0.22\\
\hhline{~------}
&Native American&All Native American defendants (6)&All non-Native American defendants (3357)&0.15&0.49&0.39\\
\hline
&\cellcolor{lightGrayTable}Under age 25&\cellcolor{lightGrayTable}Defendants under age 25 arrested on felony charges (403)&\cellcolor{lightGrayTable}Defendants age 25+ arrested on felony charges (1583)&\cellcolor{lightGrayTable}\textbf{149.2}&\cellcolor{lightGrayTable}0.55&\cellcolor{lightGrayTable}0.29\\
\hhline{~------}
&6+ priors&All defendants with 6+ priors (349)&All defendants with 0-5 priors (3014)&\textbf{125.5}&0.66&0.26\\
\hhline{~------}
&\cellcolor{lightGrayTable}Black&\cellcolor{lightGrayTable}Black male defendants (1168)&\cellcolor{lightGrayTable}Non-Black male defendants (1433) &\cellcolor{lightGrayTable}\textbf{100.9}&\cellcolor{lightGrayTable}0.44&\cellcolor{lightGrayTable}0.19\\
\hhline{~------}
&Female&Female defendants with 6+ priors (40)&Male defendants with 6+ priors (309)&46.9&0.80&0.65\\
\hhline{~------}
Separation Scan for Recommendations&\cellcolor{lightGrayTable}Male&\cellcolor{lightGrayTable}Asian and Hispanic male defendants (286)&\cellcolor{lightGrayTable}Asian and Hispanic female defendants (57)& \cellcolor{lightGrayTable}22.3&\cellcolor{lightGrayTable}0.21&\cellcolor{lightGrayTable}0.05\\
\hhline{~------}
&1 to 5 priors&Defendants under age 25 with 1 to 5 priors (227)&Defendants under age 25 with 0 or 6+ priors (366)&12.5&0.64&0.47\\
\hhline{~------}
&\cellcolor{lightGrayTable}Felony&\cellcolor{lightGrayTable}White female defendants arrested on felony charges (139)&\cellcolor{lightGrayTable}White female defendants arrested on misdemeanor charges (173)&\cellcolor{lightGrayTable}9.40&\cellcolor{lightGrayTable}0.38&\cellcolor{lightGrayTable}0.21\\
\hhline{~------}
&White&White female defendants under age 25 with no priors (31)&Non-white female defendants under age 25 with no priors (70)& 1.98&0.71&0.56\\
\hhline{~------}
&\cellcolor{lightGrayTable}Misde-meanor&\cellcolor{lightGrayTable}Native American defendants with 1 to 5 priors arrested on misdemeanor charges (2)&\cellcolor{lightGrayTable}Native American defendants with 1 to 5 priors arrested on felony charges (1)&\cellcolor{lightGrayTable}1.67&\cellcolor{lightGrayTable}1.00&\cellcolor{lightGrayTable}0.00\\
\hhline{~------}
&Age 25+&Asian defendants age 25+ arrested on felony charges (10)&Asian defendants under age 25 arrested on felony charges (1)&0.75&0.20&0.00\\
\hhline{~------}
& \cellcolor{lightGrayTable}Native American&\cellcolor{lightGrayTable}All Native American defendants (6)&\cellcolor{lightGrayTable}All non-Native American defendants (3357)&\cellcolor{lightGrayTable}0.20&\cellcolor{lightGrayTable}0.50&\cellcolor{lightGrayTable}0.30\\
\hline
&No priors&All defendants with no priors (2085)&All defendants with 1+ priors (4087)&\textbf{111.5}&0.29&0.54\\
\hhline{~------}
&\cellcolor{lightGrayTable}Age 25+&\cellcolor{lightGrayTable}Male defendants age 25+ with 0-5 priors (2867)&\cellcolor{lightGrayTable}Male defendants under age 25 with 0-5 priors (1041)&\cellcolor{lightGrayTable}\textbf{92.6}&\cellcolor{lightGrayTable}0.35&\cellcolor{lightGrayTable}0.59\\
\hhline{~------}
&Male&Male Native American defendants of age 25+ (7)&Female Native American defendants of age 25+ (2)&31.2&0.14&1.00\\
\hhline{~------}
&\cellcolor{lightGrayTable}Female&\cellcolor{lightGrayTable}Female defendants under age 25 (246)&\cellcolor{lightGrayTable}Male defendants under age 25 (1101)&\cellcolor{lightGrayTable}18.7&\cellcolor{lightGrayTable}0.38&\cellcolor{lightGrayTable}0.60\\
\hhline{~------}
Sufficiency Scan for Predictions&Misde-meanor&Female defendants arrested on misdemeanor charges (491)&Female defendants arrested on felony charges (684)& 3.51&0.26&0.41\\
\hhline{~------}
&\cellcolor{lightGrayTable}Asian&\cellcolor{lightGrayTable}Asian defendants arrested on misdemeanor charges (12)&\cellcolor{lightGrayTable}Non-Asian defendants arrested on misdemeanor charges (2190) &\cellcolor{lightGrayTable}\textbf{3.20}&\cellcolor{lightGrayTable}0.00&\cellcolor{lightGrayTable}0.38\\
\hhline{~------}
&White&White defendants under age 25 (347)&Non-white defendants under age 25 (1000) &2.36&0.49&0.58\\
\hhline{~------}
&\cellcolor{lightGrayTable}Black&\cellcolor{lightGrayTable}Black female defendants (549)&\cellcolor{lightGrayTable}Non-Black female defendants (626)&\cellcolor{lightGrayTable}2.22&\cellcolor{lightGrayTable}0.37&\cellcolor{lightGrayTable}0.34\\
\hhline{~------}
&1 to 5 priors&Black defendants of age 25+ with 1 to 5 priors (1038)&Black defendants of age 25+ with 0 or 6+ priors (1328)&2.18&0.42&0.55\\
\hhline{~------}
&\cellcolor{lightGrayTable}Hispanic&\cellcolor{lightGrayTable}All Hispanic defendants (509)&\cellcolor{lightGrayTable}All non-Hispanic defendants (5663)&\cellcolor{lightGrayTable}0.27&\cellcolor{lightGrayTable}0.37&\cellcolor{lightGrayTable}0.46\\
\hhline{~------}
&Native American&All Native American defendants (11)&All non-Native American defendants (6161)&0.17&0.45&0.46\\
\hline
&\cellcolor{lightGrayTable}Age 25+&\cellcolor{lightGrayTable}Male defendants of age 25+ with 0-5 priors (772)&\cellcolor{lightGrayTable}Male defendants under age 25 with 0-5 priors (641)&\cellcolor{lightGrayTable}\textbf{52.9}&\cellcolor{lightGrayTable}0.52&\cellcolor{lightGrayTable}0.67\\
\hhline{~------}
&No priors&All defendants with no priors (553)&All defendants with 1+ priors (2198) &\textbf{51.0}&0.46&0.67\\
\hhline{~------}
&\cellcolor{lightGrayTable}Female&\cellcolor{lightGrayTable}Female defendants with no priors (150)&\cellcolor{lightGrayTable}Male defendants with no priors (403)&\cellcolor{lightGrayTable}\textbf{35.3}&\cellcolor{lightGrayTable}0.33&\cellcolor{lightGrayTable}0.50\\
\hhline{~------}
&Misde-meanor&Male defendants with 0-5 priors arrested on misdemeanor charges (398)&Male defendants with 0-5 priors arrested on felony charges (1015)&28.2&0.52&0.61\\
\hhline{~------}
Sufficiency Scan for Recommendations&\cellcolor{lightGrayTable}1 to 5 priors&\cellcolor{lightGrayTable}Male defendants of age 25+ with 1 to 5 priors (595)&\cellcolor{lightGrayTable}Male defendants of age 25+ with 0 or 6+ priors (981)&\cellcolor{lightGrayTable}\textbf{26.8}&\cellcolor{lightGrayTable}0.54&\cellcolor{lightGrayTable}0.70\\
\hhline{~------}
&Male&Male Native American defendants of age 25+ (4)&Female Native American defendants of age 25+ (2)&14.1&0.25&1.00\\
\hhline{~------}
&\cellcolor{lightGrayTable}Hispanic&\cellcolor{lightGrayTable}All Hispanic defendants (141)&\cellcolor{lightGrayTable}All non-Hispanic defendants (2610)&\cellcolor{lightGrayTable}2.40&\cellcolor{lightGrayTable}0.56&\cellcolor{lightGrayTable}0.63\\
\hhline{~------}
&6+ priors&Asian defendants with 6+ priors (1)&Asian defendants with 0-5 priors (6) &0.41&0.00&0.83\\
\hhline{~------}
&\cellcolor{lightGrayTable}White&\cellcolor{lightGrayTable}White female defendants under age 25 (57)&Non-white female defendants under age 25 (110)&\cellcolor{lightGrayTable}0.40&\cellcolor{lightGrayTable}0.39&\cellcolor{lightGrayTable}0.47\\
\hhline{~------}
&Black&Black defendants of age 25+ with 0-5 priors (581)&Non-Black defendants of age 25+ with 0-5 priors (404)& 0.34&0.50&0.52\\
\hline
\caption{Full table of results for COMPAS case study. Each of the four variants of CBS was run using each observed attribute value as the protected class. Detected subgroup $S^\ast$ of the protected class and corresponding (comparison) subgroup of the non-protected class; numbers of defendants for each subgroup are shown in parentheses. All runs with log-likelihood ratio score $F(S^\ast) > 0$ are shown, sorted in descending order by score for each method. Separation scan for predictions: ``observed rate'' is average predicted probability of reoffending, $\mathbb{E}[P_i]$, for defendants who did not reoffend ($Y_i = 0$). Separation scan for recommendations: ``observed rate'' is false positive rate, i.e., proportion of individuals predicted as ``high risk'' ($P_{i,bin}=1$) for defendants who did not reoffend ($Y_i = 0$). Sufficiency scan for predictions: ``observed rate'' is proportion of reoffending individuals ($Y_i = 1$), controlling for predicted risk. Sufficiency scan for recommendations: ``observed rate'' is positive predictive value, i.e., proportion of reoffending individuals ($Y_i = 1$) for defendants who were predicted as ``high risk'' ($P_{i,bin} = 1$). Bolded scores are statistically significant with p-value <.05 measured by permutation testing, as described in Appendix~\ref{permutation_testing}.}
\label{compas_full_results}
\end{longtable}

\subsection{Considerations and Limitations of COMPAS Data and Fairness Definitions in our COMPAS Case Study} \label{compas_considerations}

Following the initial investigation by ProPublica about fairness issues in COMPAS risk predictions~\citep{angwin2016machine}, ProPublica's COMPAS dataset has been used as a benchmark in the fairness literature. While we use the COMPAS data because of its familiarity and supporting research, we also note the value of alternative framings of the evaluation of automated decision support tools in the criminal justice systems, such as examining the risks that the system poses to defendants rather than the risk of the defendants to public safety~\cite{mitchell2021algorithmic,meyer2022flipping, green2020false}. Beyond the implications of the traditional framing of pre-trial risk assessment tools, there have been specific critiques of the COMPAS data that range from questioning the accuracy of the sensitive attributes (specifically race), noting missing features in the ProPublica dataset that the COMPAS creators claim are important for score calculations, and most importantly, a lack of evaluation of the biases that exist in the outcome variable of whether a defendant is rearrested within two years of arrest.\footnote{See, for example, Fabris et al., ``Algorithmic fairness datasets: the story so far'', \emph{Data Mining and Knowledge Discovery} 36(6), 2074-2152, 2022.} Given that certain types of individuals are arrested at a higher rate than others, the outcome variable of re-arrest most likely under- and over-represents certain subpopulations of defendants.

In our COMPAS case study, for the separation scans, we search for subgroups of the protected class with the most significant \emph{increase}, either in the probabilistic predictions or in the probability that the binarized recommendation equals 1, conditional on the defendant's covariates. Moreover, we perform value-conditional scans, focusing specifically on the subset of defendants who did not reoffend ($Y_i = 0$). 
For the separation scan for recommendations, this results in CBS detecting subgroups of the protected class for whom the \emph{false positive rate} is most significantly increased. For the sufficiency scans, we search for subgroups of the protected class with the most significant \emph{decrease} in the observed rate of reoffending, conditional on the defendant's covariates and their COMPAS prediction or recommendation. For the sufficiency scan for recommendations, we also perform a value-conditional scan. We focus specifically on the subset of defendants who were predicted to be ``high risk'' by COMPAS ($P_{i,bin}=1$) because this labeling could negatively impact the defendant, e.g., by decreasing their likelihood of pre-trial release. This results in CBS detecting subgroups of the protected class for whom the \emph{false discovery rate} is most significantly increased.  These fairness definitions neglect bias detection for defendants who reoffend (for separation scans) and defendants who are not flagged as high-risk (for sufficiency scan for recommendations).  These choices were made to ensure our ability to verify our findings based on previous research on COMPAS which commonly focus on similar fairness violations to those used in our case study.  With that said, we strongly encourage auditing for predictive biases that affect reoffending defendants and low-risk defendants as well, if using CBS to audit an algorithmic risk assessment tool in practice. For example, auditing for the increased probability of being flagged as high-risk for reoffending defendants could help to uncover subpopulations that are over-prosecuted in comparison to other populations of reoffending defendants. Therefore, expanding the fairness definitions used to audit pre-trial risk assessment tools for biases could have beneficial findings\nocite{ramchand2006racial}.

\end{document}